\documentclass{article}
 
\usepackage{PRIMEarxiv}
\usepackage[numbers,sort&compress]{natbib}
\usepackage[utf8]{inputenc} % allow utf-8 input
\usepackage[T1]{fontenc}    % use 8-bit T1 fonts
\usepackage{hyperref}       % hyperlinks
\usepackage{url}            % simple URL typesetting
\usepackage{booktabs}       % professional-quality tables
\usepackage{amsfonts}       % blackboard math symbols
\usepackage{nicefrac}       % compact symbols for 1/2, etc.
\usepackage{microtype}      % microtypography
\usepackage{fancyhdr}       % header
\usepackage{graphicx}       % graphics
\usepackage{amsmath}
\usepackage{amsthm}
\usepackage{amssymb}
\usepackage{xcolor}
 
\usepackage{amsmath,amsfonts,amsthm,amssymb,bm}

\def\eqref#1{equation~\ref{#1}}
\def\1{\bm{1}}

\DeclareMathAlphabet{\mathsfit}{\encodingdefault}{\sfdefault}{m}{sl}
\SetMathAlphabet{\mathsfit}{bold}{\encodingdefault}{\sfdefault}{bx}{n}

\usepackage{xcolor}
\usepackage{annotate-equations}
\usepackage{hyperref}
\hypersetup{
  colorlinks=true,
  linkcolor=green!30!black,
  linktoc=all,
  citecolor=blue,
  bookmarksnumbered=true,
  pdfproducer={},
  pdfinfo={CreationDate={2025},ModDate={2025}}, % UPDATE
  pdfauthor={},
  pdfkeywords={},
  pdfsubject={},
  pdfcreator={},
  pdflang={}
}
\usepackage{cleveref}
\usepackage{mathtools,pifont,enumitem}
\usepackage[framemethod=TikZ]{mdframed}
\usepackage{color,fancybox,graphicx,subcaption}

\usepackage{xurl}
\usepackage{scalerel}
\usetikzlibrary{math}
\usepackage{bbding}
\usepackage{array}
\usepackage{dsfont}

\usepackage{empheq}
\usepackage{stackengine}

\newtheorem{assumption}{Assumption}
\newtheorem{theorem}{Theorem}[section]
\newtheorem{lemma}[theorem]{Lemma}
\newtheorem{definition}[theorem]{Definition}

\newtheorem{corollary}[theorem]{Corollary}

\newtheorem{proposition}[theorem]{Proposition}

\crefname{section}{Section}{Sections}
\crefname{theorem}{Theorem}{Theorems}
\crefname{assumption}{Assumption}{Assumptions}
\crefname{lemma}{Lemma}{Lemmas}
\crefname{definition}{Definition}{Definitions}
\crefname{conjecture}{Conjecture}{Conjectures}
\crefname{corollary}{Corollary}{Corollaries}
\crefname{construction}{Construction}{Constructions}
\crefname{claim}{Proposition}{Propositions}
\crefname{observation}{Observation}{Observations}
\crefname{proposition}{Proposition}{Propositions}
\crefname{fact}{Fact}{Facts}
\crefname{question}{Question}{Questions}
\crefname{problem}{Problem}{Problems}
\crefname{remark}{Remark}{Remarks}
\crefname{example}{Example}{Examples}
\crefname{equation}{Equation}{Equations}
\crefname{appendix}{Appendix}{Appendices}
\crefname{algorithm}{Algorithm}{Algorithms}
\crefname{model}{Model}{Models}
\crefname{figure}{Figure}{Figures}

\AfterEndEnvironment{definition}{\noindent\ignorespaces}
\AfterEndEnvironment{assumption}{\noindent\ignorespaces}
\AfterEndEnvironment{lemma}{\noindent\ignorespaces}
\AfterEndEnvironment{theorem}{\noindent\ignorespaces}
\AfterEndEnvironment{proposition}{\noindent\ignorespaces}
\AfterEndEnvironment{fact}{\noindent\ignorespaces}
\AfterEndEnvironment{question}{\noindent\ignorespaces}
\renewcommand{\eqref}[1]{(\ref{#1})}

\providecommand{\Smi}{}\renewcommand{\Smi}{S^{\setminus i}}

\providecommand{\dNN}{}\renewcommand{\dNN}{d_{\mathrm{NN}}}

\title{
Loss Curvature as a Sampling-Free Proxy for Memorization in Diffusion Models
}
 
\author{
  Shashank Sule$^{*,1}$ \quad
  Xinyue Yu$^{*,1}$ \quad
  Boonam Shin$^{2}$ \\ \\
  \bfseries Yuxuan Liu$^{1}$ \quad
  Abolfazl Hashemi$^{2}$ \quad
  Hayden Schaeffer$^{1}$
}

\newcommand{\authornote}[2]{%
  \begingroup\renewcommand{\thefootnote}{#1}\footnotetext{#2}\endgroup}

\begin{document}
\maketitle
\authornote{$\ast$}{Equal contribution.}
\authornote{1}{Department of Mathematics, University of California, Los Angeles, CA 90095, USA.}
\authornote{2}{Elmore Family School of Electrical and Computer Engineering, Purdue University, West Lafayette, IN 47907, USA.}

\begin{abstract}
  Understanding the mechanism behind why diffusion models memorize their training data remains an open question in generative modeling. In this paper, we examine this question from two perspectives: (a) the model-level perspective concerning the onset of memorization in a model during training, and (b) the sample-level perspective concerning the proclivity of an individual sample to be memorized by a given model. We show that the so-called input loss curvature, or the Frobenius norm of the Hessian of the diffusion model's training loss, is an effective metric for quantifying both perspectives on memorization. In particular, we empirically demonstrate that input curvatures averaged across training and validation samples can identify when a model is likely to start regurgitating its training data when tracked over training time. However, this metric is not sufficient for measuring an individual sample's likelihood of being memorized. To address this perspective, we introduce a counterfactual definition of memorization in diffusion models, extending the analogous notion from supervised learning proposed by [Feldman, 2019]. We prove that this counterfactual memorization is controlled by the input curvature of the leave-one-out model, thus establishing a direct theoretical link between a sample's local loss landscape and its generation likelihood. We validate our derived bounds through several empirical experiments and show that input-dependent curvature provides a proxy for memorization. 
\end{abstract}

\section{Introduction}
\label{sec:intro}

Diffusion and score-based models are state-of-the-art generative models across domains from vision to drug discovery
\cite{sohl2015deep,ho2020denoising,song2019generative,song2021scorebased,rombach2022high, nichol2021improved, abramson2024accurate, ramesh2021zero}.
Their appeal lies in generating novel yet realistic data, suggesting they estimate the data distribution from finitely many samples.
Yet diffusion models sometimes regurgitate training samples, a failure mode called \emph{memorization}, which poses privacy and copyright risks
\cite{carlini2023extracting,somepalli2023diffusion,somepalli2023understanding,gu2023memorization}.
Existing work studies memorization at two levels: 
\begin{enumerate}
    \item \emph{Model-level memorization} asks how closely the trained score matches the minimizer of the empirical denoising score matching loss
\cite{vincent2011connection, baptista2025memorization}, namely the score of the Gaussian KDE of the training set, whose sampling reproduces the training data exactly.
Tracking training time $\tau$, models first generalize (at $\tau_{gen}$) and later memorize (at $\tau_{mem}$)
\cite{biroli2023generative,biroli2024dynamical,yoon2023diffusion,achilli2024losing}, but these time-scales are known only asymptotically,
and measuring the memorized fraction $f_{mem}(\tau)$ requires costly sampling.
\item \emph{Sample-level memorization} asks which training points are copied. Previous work attributes sample-level memorization to the presence of duplicates 
\cite{carlini2023extracting,somepalli2023understanding,zhang2023counterfactual,carlini2021extracting}. More generally, sample-level memorization refers to the influence of an individual sample on a model. This connects to to the long-tail theory of \cite{feldman2020does}, which motivates the counterfactual memorization measure of \cite{feldman2020neural}
(see also \cite{zhang2017understanding}).
Extending this measure to diffusion models is nontrivial because the training objective (denoising loss) differs from the test objective (likelihood),
and generalization in diffusion models has its own geometric structure \cite{kadkhodaie2024generalization}.
\end{enumerate}

These perspectives motivate \emph{sampling-free} proxies of memorization which can ideally be measured during training. We propose using \emph{input loss curvature} from \cite{ravikumar2024curvature} given by
\begin{align}
  \text{Curv}_{\phi}(x) = \mathbb{E}_{\phi \sim \mathcal{A}(S)}\left[ \| \nabla^{2}_{x} \ell(\phi, x) \|^{2}_{F}\right], \label{eq: loss curvature}
\end{align}
as such a proxy. We use it for two reasons: (i) it is the leading order term that bounds Feldman's memorization score, is among the strongest membership-inference signals, and estimates differential privacy
\cite{ravikumar2024unveiling,ravikumar2024curvature,garg2024memorization,shokri2017membership,carlini2022membership,dwork2006calibrating};
(ii) the Hessian of the potential $-\log p(x)$ governs the local geometry of sampling, with higher curvature confining Langevin trajectories near a minimizer for longer
\cite{bovier2015metastability,berglund2013kramers}. However, during training it is prohibitively expensive to compute these negative log likelihoods--instead, we investigate whether the second order information on the loss can help predict the properties implied by the second order information of the negative log likelihood. 

\paragraph{Contributions.}\begin{enumerate}[leftmargin=*,itemsep=1pt,topsep=2pt,label=(\alph*)]\item \emph{Curvature over training image tracks memorization.} We demonstrate that on ImageNet $64\times 64$ and CelebA $32 \times 32$,       $\text{Curv}_{\phi}(x)$ averaged over the training set corresponds to the $f_{mem}$ score and rises as the model overfits. Moreover,       at the onset of memorization, samples that are copied (i.e appear to be in the training set) have higher curvature than samples that      are not copied.  \item \emph{Held-out curvature separates from training curvature during memorization.} While curvature over training images rises, the curvature       over the validation set starts separating from the training curvature at the onset of memorization. Surprisingly, validation curvature       hits a minimum during training and starts rising as memorization proceeds, while validation loss continues to decrease. The local       minimum of validation curvature provides an early-stopping measure for diffusion models that indicates the prevalence of model-level memorization.  \item \emph{A counterfactual notion of memorization.} We illustrate that while curvature indicates model-level memorization,       it cannot signal the copy frequency of individual samples. To quantify such individual sample influence, we introduce a new notion of memorization,       termed counterfactual memorization given by the log-odds of sampling $x_i$ with and without its inclusion in the training data.  \item \emph{Counterfactual memorization is bounded by curvature.} By lifting the differences in log-likelihoods to model losses via a decomposition given in \cite{huang2021variational}, we quantify the influence of training dynamics on counterfactual memorization. In particular, we demonstrate that when the model starts overfitting, the log-likelihood difference is captured at the leading order by the curvature of the leave-one-out model and the nearest neighbour distance.  \end{enumerate}

\section{Background}\label{sec: background}
\subsection{Diffusion Models} \label{subsec: diffusion models}

A diffusion model is a generative algorithm for sampling from an unknown density $p$ by transforming noise into clean data. To learn such a model we first establish the following forward process:
\begin{align} dY_t = f(t,Y_t),dt + g(t),dW_t, \qquad Y_0 = x \sim p, \label{eq: forward-diffusion} 
\end{align} 
When $f(t,y) = f(t)y$, the transition kernel is Gaussian: $Y_t\mid Y_0 = x \sim \mathcal{N}(m(t)x,\sigma(t)^2 I)$. The schedules $(m,\sigma)$ and the coefficients $(f,g)$ thus determine each other: 
\begin{align} f(t) = \frac{m'(t)}{m(t)}, \qquad g(t)^2 = \frac{d}{dt}\sigma(t)^2 - 2\frac{m'(t)}{m(t)},\sigma(t)^2. \label{eq: fg-from-msigma} \end{align}
Anderson's theorem \cite{anderson1982reverse} gives the backward SDE, whose law at time $t$ is that of $Y_{T-t}$: 
\begin{align}
d\overline{X}_t &= \big(\bar g(t)^2\nabla\log p_{T-t}(\overline{X}_t) - \bar f(t,\overline{X}_t)\big),dt \\
&+ \bar g(t),dB_t \label{eq: reverse-sde} 
\end{align} 
where $\bar f(t,x) = f(T-t,x)$ and $\bar g(t) = g(T-t)$. Thus, given $f$ and $g$, sampling from $p$ reduces to knowing the score $\nabla\log p_t$ and running the backward SDE from $t=0$ to $t=T$. Via the Tweedie-Miyasawa formula 
\begin{align}
    \mathbb{E}[X \mid Y_t = y] = \frac{y + \sigma^2(t)\nabla \log p(t,y)}{m(t)}, \label{eq: tweedie-miyasawa}
\end{align}
we can instead parameterize the denoiser $\mathbb{E}[X \mid Y_t = y]$ and recover the score by rearranging \eqref{eq: tweedie-miyasawa}. We therefore train a model $D_{\phi}$ by MMSE denoising:  
\begin{align}
    \ell(\phi,x) := \mathbb{E}_{t \in [0,T], y_t \sim Y_t} \left[\lambda(t)\|x - D_{\phi}(y_t, t)]\|_{2}^{2} \mid Y_0 = x\right]. \label{eq: denoising loss}
\end{align}
Here $\lambda(t)$ is a choice of weighting function. Substituting the learned score $s_{\phi} = \sigma(t)^{-2}(m(t)D_{\phi} - y)$ gives the \emph{plug-in SDE}: 
\begin{align}
    dX_t = (\overline{g}(t)^2 s(t, X_t) - \overline{f(t, X_t)})\,dt + \overline{g}(t)\,dB_t. \label{eq: plug-in sde}
\end{align}

The plug-in SDE has density $p_{\phi}(y,t)$. At generation time we integrate \eqref{eq: plug-in sde} forward, obtaining the terminal distribution $p_{\phi}(\cdot, T)$ as an approximation to the true distribution $p$. 

\subsection{DDPMs} In this paper, we use Denoising Diffusion Probabilistic Models (DDPMs) \cite{ho2020denoising} corresponding to taking $f(t) = -\tfrac{1}{2}\beta(t)$ and $g(t) = \sqrt{\beta(t)}$ for some non-negative rate function $\beta$. Writing $\bar\alpha(t) := \exp\left(-\int_0^t \beta(u)\,du\right)$, one finds $\Phi(t) = \sqrt{\bar\alpha(t)}$ and, using $\frac{d}{ds}\bar\alpha(s)^{-1} = \beta(s)\bar\alpha(s)^{-1}$,
\begin{align}
    m(t) = \sqrt{\bar\alpha(t)}, \qquad \sigma(t)^2 = 1 - \bar\alpha(t). \label{eq: vp closed form}
\end{align}
To sample from DDPMs, we use an accelerated non-Markovian sampler given in \cite{song2021denoising}. More generally, note that the sampling process via \eqref{eq: plug-in sde} is implemented by discretizing the SDE via Euler-Maruyama, leading to a stochastic sampling trajectory. A deterministic sampling, using the probability flow ODE or flow-matching may also be used \cite{song2021scorebased, lipman2023flow}. In Appendix \ref{app: experimental details}, we provide experimental details and additional details on flow matching. 

\subsection{Log-Likelihoods via the Denoising Loss}
Diffusion models are trained by using the loss \eqref{eq: denoising loss}. However, at test time, their performance is measured by the fidelity of $p(\cdot, T)$ to the true distribution $p$. Therefore, their training and test tasks are rather different. However, an evidence lower bound (ELBO) shows that the denoising loss bounds the negative log-likelihood of the distribution $p_{\phi}(\cdot, T)$. In fact, \cite{huang2021variational} described the gap between the ELBO and $-\log p_{\phi}(\cdot, T)$ exactly under the likelihood weighting $\lambda(t) = g^2(t)$: 

\begin{proposition}[Theorem 4 in \cite{huang2021variational}] Let $\ell(\phi,x)$ be the denoising loss \eqref{eq: denoising loss}, $p_{\phi}(x, T)$ the distribution of the plug-in SDE
\eqref{eq: plug-in sde} at time $T$ and $s_{\phi}$ the learned score function in \eqref{eq: plug-in sde}. Then 
\small
\begin{align}
    &-\log p_{\phi}(x,T) - \ell(\phi,x) = C \notag \\
    &+ \underbrace{\int_{0}^{T}\mathbb{E}\left[g^{2}(t)\|s_{\phi}(Y_t, t) - \nabla \log_{\phi}(Y_t, T-t)\|_{2}^{2} \mid Y_0 = x\right]\,dt}_{\text{Gap}_{\phi}(x)}, \label{eq: elbo gap}
\end{align}
\normalsize
where $C$ is a constant independent of $x$ or $\phi$ but dependent on the data dimension $d$, and the noise schedules $f,g$. 
\end{proposition}

The RHS in \eqref{eq: elbo gap} is purely a \emph{self-consistency} gap, i.e it measures how close the score of the reversed plug-in SDE is to the original forward SDE. 

\subsection{Memorization} \label{subsec: memorization}

Here we summarize the framework of \cite{feldman2020does,feldman2020neural}. Let $\mathcal{A}$ be a training algorithm, possibly randomized, and write $h^\phi_S$ for the model with parameters $\phi\sim\mathcal{A}(S)$. Let $\ell(h^\phi_S,z)$ be its loss on an example $z$, and let $S^{\setminus i} = S\setminus{z_i}$. The memorization of $z_i$ is the change in expected loss at $z_i$ caused by its inclusion in the training set, 
\begin{align} 
M(z_i,S) &= \Big|,\mathbb{E}_{\phi\sim\mathcal{A}(S)}[\ell(h^\phi_S,z_i)] \nonumber\ - \mathbb{E}_{\phi'\sim\mathcal{A}(S^{\setminus i})}[\ell(h^{\phi'}_{S^{\setminus i}},z_i)]\Big|. 
\end{align} 
A large value means the model only fits $z_i$ because it was trained on $z_i$.
\section{Empirical Results}
\label{sec: empirical results}
We begin with an empirical investigation of curvature as a proxy for memorization in diffusion models. To do so, we train DDPMs to sample two datasets: (1) $64 \times 64$ RGB ImageNet-1K \cite{deng2009imagenet, russakovsky2015imagenet}, and (2) $32 \times 32$ greyscale CelebA images, also used in \cite{biroli2023generative}). To further isolate the memorization phenomenon, we train six distinct DDPMs on datasets of size $n= 2^{k}$ for $7 \leq k \leq 12$ for ImageNet-1k and two distinct DDPMs on datasets of size $n=128, 10240$ for CelebA, noting that for CelebA, $n=128$ corresponds to a memorizing model while $n=10240$ corresponds to a generalizing model. Across the training time $\tau$, we track the mean curvature over the training samples and validation samples: 
\small
\begin{align}
    \text{Avg. train curvature} &= \mathbb{E}_{x \sim \text{train}}[\|\nabla^{2}_{x}\ell(\phi(\tau), x)\|^{2}_{F}] \label{eq: avg train curvature}, \\
        \text{Avg. val curvature} &= \mathbb{E}_{x \sim \text{val.}}[\|\nabla^{2}_{x}\ell(\phi(\tau), x)\|^{2}_{F}]. \label{eq: avg val curvature}
\end{align}
\normalsize
We also generate $10,000$ samples to compute the $f_{mem}$ score. Following \cite{gu2023memorization, yoon2023diffusion, bonnaire2025diffusion}, a generated sample $x$ is considered memorized if 
\begin{equation}
\left[\frac{\|x - a^{\mu_1}\|_2}{\|x -a^{\mu_2}\|_2}<k\right],
\end{equation}
where $a^{\mu_1}$ and $a^{\mu_2}$ are the nearest and second nearest neighbors in the training set to $x$. We say that a training sample $x_i$ is \emph{copied} if it appears in the generated set. Its \emph{copy count} $F(i)$ refers to the frequency of occurrences of $x_i$ in the generations. In Appendix \ref{app: generations} we visualize the samples generated from our various models. Here we assess the input loss curvature of our models during training. 

\textbf{Evolution of input curvature during training.} 
In Figure~\ref{fig: celebA}, we report our experiments on CelebA. We find that: 

\begin{enumerate} 
\item Mean training curvature rises over training together with $f_{\mathrm{mem}}(\tau)$ (panel~(C)). The rise alone does not mark the onset: training curvature is already increasing while $f_{\mathrm{mem}}$ is still flat. 
\item The onset shows up in the validation curvature. It stays with the training curvature and then diverges, in both panels~(A) and~(B), at the step where $f_{\mathrm{mem}}(\tau)$ starts to accelerate. The same split appears for the generalizing model (panel~(D)), where memorization is delayed until $\tau \approx 10^{5}$. This motivates defining the relative change in avg. train and validation curvatures as a stopping criterion: 
{\footnotesize
\begin{align}
        &\Delta \text{Curv}(\tau) \label{eq: stopping criteria} \\
        &:= \frac{|\text{Avg. train curvature} - \text{Avg. val. curvature}|}{|\text{Avg. train curvature}|} \notag
\end{align}}

we may consider stopping training when $\Delta \text{Curv}(\tau) \geq \varepsilon$, where $\varepsilon$ modulates the threshold required to limit memorization. Notably this estimator \emph{does not need} to generate any samples during training, saving significant computational cost. 
\item Once memorization is heavy, every training image is reproduced but at different frequencies. However, the curvature of an individual image does not predict its copy count (Panel~(F)). \end{enumerate}

In Figure~\ref{fig: imagenet}, we repeat these experiments for ImageNet, observing patterns consistent with CelebA. In Panel (F), we confirm that $f_{mem}(\tau)$ rises and reaches 100\% by $\tau \approx 10^6$ for every model except $n = 4096$. The validation loss, a zeroth-order quantity, falls throughout training for every $n$ (Panel (E)), so the held-out loss alone does not detect memorization. However, in panel (A), we find that input curvature for training samples rises with $\tau$, but the average validation curvature \emph{reaches a local minimum} and then rises once the model begins to memorize heavily. In Panel (D), we show that $\Delta \text{Curv}(\tau)$ reaches a local maximum at around $\tau = 10^5$, where most models are memorizing. 

Appendices \ref{app: additional} and \ref{app:imagenet-ddpm} provide additional numerical experiments.

\begin{figure*}[t]
    \centering
    \includegraphics[width=\linewidth]{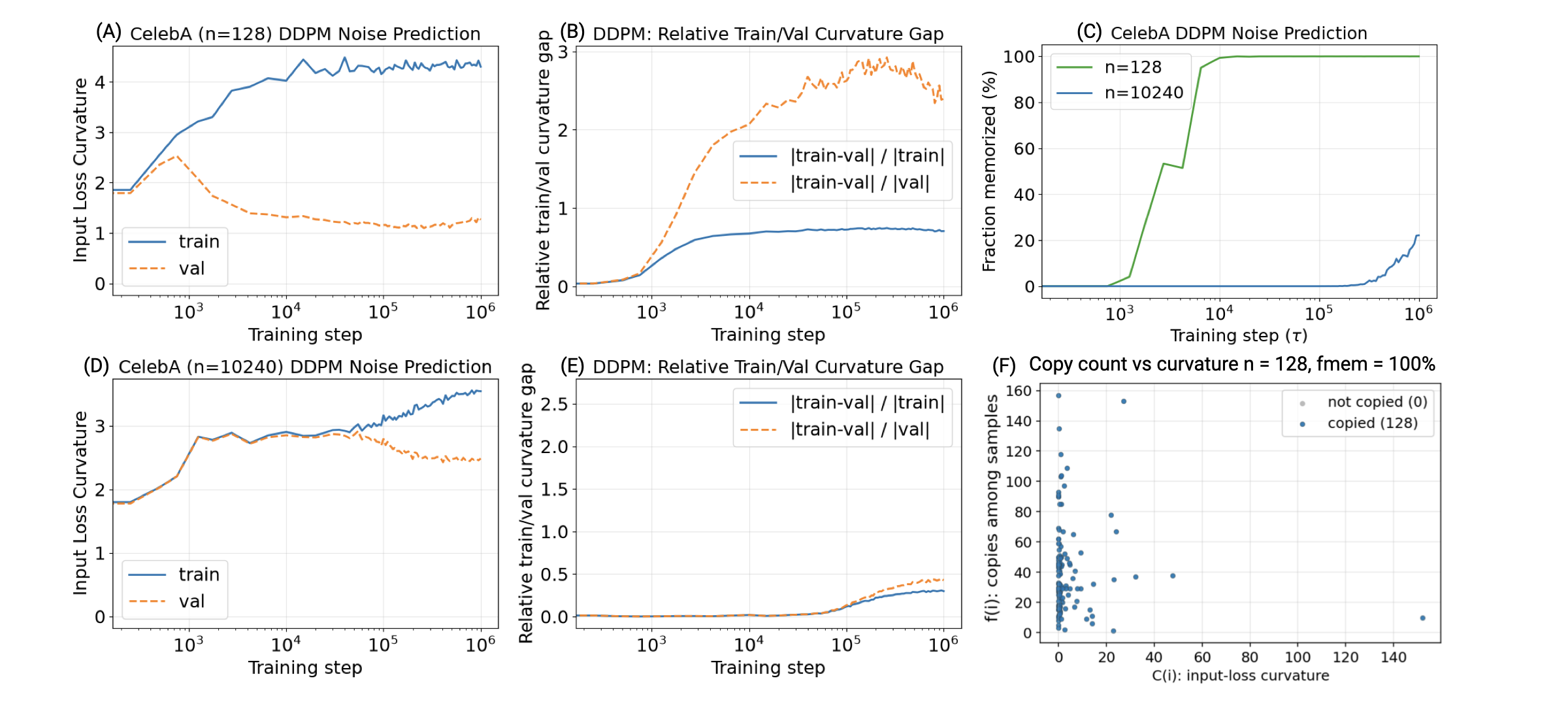}
    \caption{(A) Train vs. validation input curvature for CelebA, n=128, (B) Val. curvature starts to diverge from train curvature when $f_{mem}(\tau)$ starts accelerating, (C) $f_{mem}(\tau)$ values for $n=128, 10240$ on CelebA, (D) and (E) Same experiment for CelebA, n=10240. Note that this model undergoes generalization, and memorization sets in much later at $\tau \approx 10^5$, (F) Curvature of an individual sample is not a good predictor of its copy count.}
    \label{fig: celebA}
\end{figure*}

\begin{figure*}[t]
    \centering
    \includegraphics[width=\linewidth]{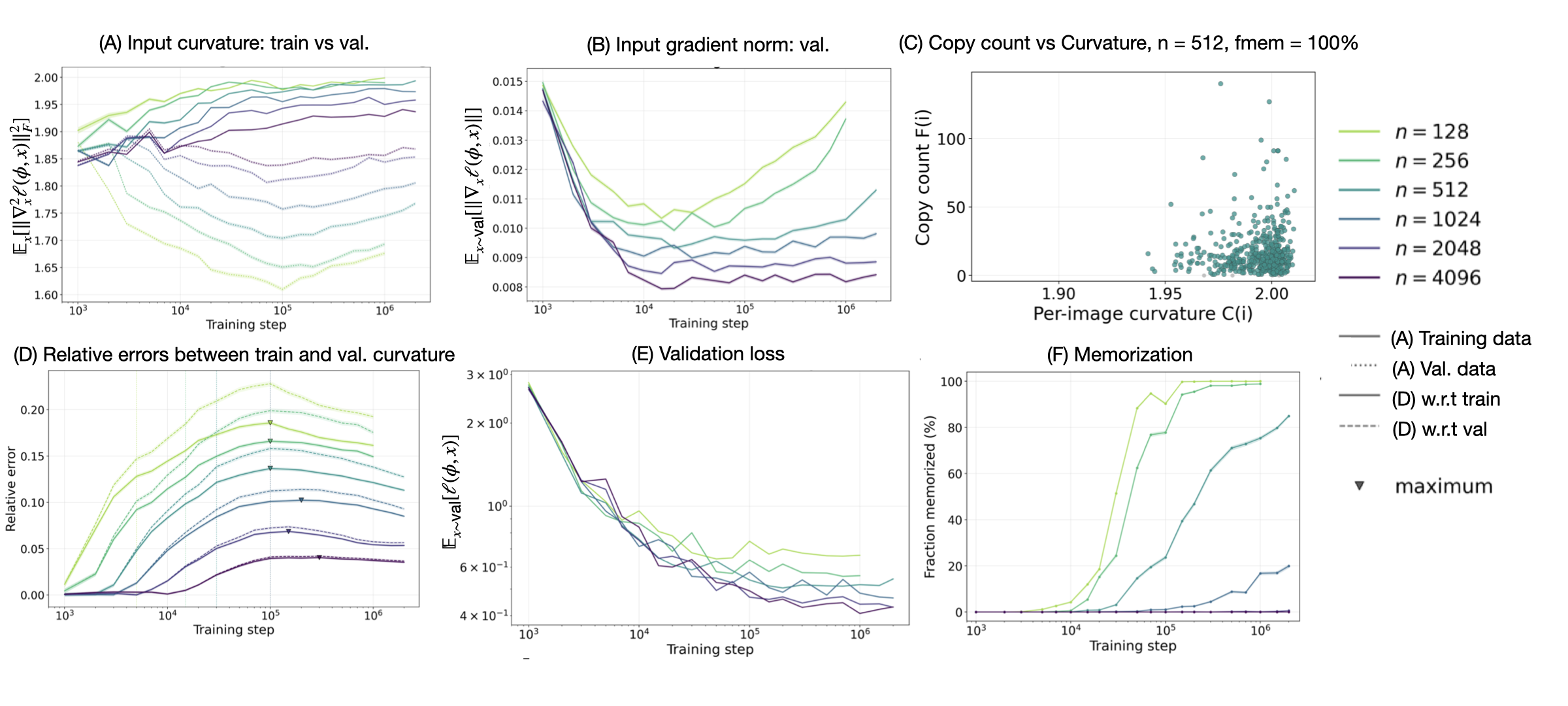}
    \caption{(A) Input \eqref{eq: avg train curvature} curvatures consistently rise but validation \eqref{eq: avg val curvature} curvatures reach a local minimum at approximately peak memorization for DDPMs on different data sizes, (B) Sizes of gradients of validation loss also display local minima, (C) Input curvature on training data tracks the largest frequency of memorized samples, (D) Relative error between train and validation curvatures exhibits a local maximum, leading to an early stopping criterion, (E) Memorization cannot be detected just based on validation loss, (F) Memorization rises as models train further.}
    \label{fig: imagenet}
\end{figure*}

% In Figure \ref{fig: memorized samples have higher curvatuer than non-memorized samples} we report training progress, noting that memorization commences ($f_{mem} > 1 \%$) at around $\tau \approx 10^{3}$ for $n=128$ and arises progressively later for larger $n$. Notably, as $f_{mem}$ rises, so does the average training curvature \eqref{eq: avg train curvature}. However, while 
% \begin{figure}
%     \centering
%     \includegraphics[width=\linewidth]{fig1.png}
%     \caption{The curvature \eqref{eq: avg train curvature} of memorized training points (orange) vs. non-memorized training points (blue).}
%     \label{fig: memorized samples have higher curvatuer than non-memorized samples}
% \end{figure}

\section{Sample-Level Memorization}
\label{sec: sample-level mem}
Our results in Section \ref{sec: empirical results} show that while curvature, aggregated over either training or held-out sets, is a good indicator of model-level memorization, it is not a good signal for whether a given sample will be memorized more than another one. To address this gap, we introduce a counterfactual memorization, analogous to what was defined in \cite{zhang2023counterfactual}. Before we begin, we establish the relevant notation. 

\subsection{Notation}
The notation will be borrowed from \cite{ravikumar2024unveiling}. Let $\mathcal{X}$ be the space of inputs, $\mathcal{Y}$ the space of outputs, and $h$ a model $h: \mathcal{X} \to \mathcal{Y}$. The joint distribution in the input-output pair $z = (x,y)$ will be denoted $p_Z$ and its $X$ and $Y$ marginals will be the input and output distributions. The model will be drawn from some parametric function class $\mathcal{C}$ parameterized by $\Phi \subseteq \mathbb{R}^{P}$ where $P$ represents the total number of parameters. To mark its dependence on parameters, the model will therefore be denoted as $h^{\phi}$. In practice, we draw a finite number of samples $S = \{z_i\}_{i=1}^{n} = \{(x_i,y_i)\}_{i=1}^{n}$ from $\mathcal{S}^{n} = (\mathcal{X} \times \mathcal{Y})^n$ and \emph{train} to fit a loss $c: \mathcal{Y} \times \mathcal{Y}'$ on the data. We assume the $n$-data points are sampled i.i.d; therefore $p_Z$ induces a distribution (via independence) on $\mathcal{S}^{n}$, termed $p^{n}_{Z}$. Given $S \sim p^{n}_{Z}$, the training process is quantified by some algorithm $\mathcal{A}: \mathcal{S}^{n} \to \mathcal{M}$, where $\mathcal{M} \subseteq \mathcal{P}_2(\Phi)$ and $\mathcal{P}_{2}(\Phi)$ is the class of square integrable random m on $\Phi$. A model trained with data $S$ with parameters $\phi$ is thus denoted $h^{\phi}_{S}$. We may define the loss over the model $h^{\phi}_{S}$ and $z$ as 
\begin{align}
    \ell(h^{\phi}_{S}, z) = c(h^{\phi}_{S}(x), y). 
\end{align}
Wherever necessary we will abuse notation and identify the function $\phi, z \to \ell(h^{\phi}_{S}, z)$ with $\ell$. Given a dataset $S$ we will also define the losses on $S$ as follows: 
\begin{align}
    \ell(\phi, S) := \frac{1}{|S|} \sum_{i=1}^{|S|}\ell(h^{\phi}_{S}, z_i). 
\end{align}
We also set $\ell^{*}(S) := \inf_{\phi \in \Phi} \ell(\phi, S)$.

\subsection{The Counterfactual Log-Odds of Sampling}
We may simplify the notation above in the specific case of diffusion models. In particular, the training objective \eqref{eq: denoising loss} is \emph{unsupervised}, and thus $z_i = x_i$. Moreover, we assume $h^{\phi}$ is produced via score matching/denoising and that it is then used to drive the SDE \eqref{eq: plug-in sde}. We now define the counterfactual memorization: 

\begin{definition}
\label{def: sampling time memorization}
    Let $z_i \in S$ be a training example and $\phi \sim \mathcal{A}(S), \phi' \sim \mathcal{A}(S^{\setminus i})$ be parameter distributions after optimization on the datasets $S, S^{\setminus i}$. Let $p_{\phi}(\cdot, T), p_{\phi'}(\cdot, T)$ be the time marginals of the distributions produced by driving the plug-in SDE \eqref{eq: plug-in sde}. Then the \emph{sampling time memorization} of $z_i$ is given by: 
    \small
    \begin{align}
        m(z_i, S) := \mathbb{E}_{\phi, \phi'}[\log p_{\phi}(z_i, T) - \log p_{\phi'}(z_i,T)]. \label{eq: sampling time mem}
    \end{align}
    \normalsize
\end{definition}
Note that $m(z_i, S)$ represents the log-odds of sampling $z_i$ with and without its presence in the training data. Large $m$ implies that the model would be much less likely to sample $z_i$ if it wasn't included in the dataset. We also introduce a difference in the counterfactual \emph{losses}, borrowed from \cite{ravikumar2024curvature, ravikumar2024unveiling, feldman2020does, feldman2020neural}: 
\begin{align}
    M(z_i, S) &= \mathbb{E}_{\phi \sim \mathcal{A}(S)}[\ell(h^{\phi}_{S}, z_i)] \\
    &- \mathbb{E}_{\phi' \sim A(S^{\setminus i})}[\ell(h^{\phi'}_{S^{\setminus i}}, z_i)]. \label{eq: train time memorization}
\end{align}
We now use the decomposition \eqref{eq: elbo gap} to write $m$ in terms of $M$: 
\begin{proposition}
    Let $\phi, \phi'$ be parameter distributions as defined in \ref{def: sampling time memorization} and $m$, $M$ be defined as in \eqref{eq: sampling time mem} and \eqref{eq: train time memorization}. Fix $x_i \in S$. Then: 
    \begin{align}
        &-\mathbb{E}_{\phi'}[\text{Gap}_{\phi'}(x_i)] - M(x_i, S) \leq m(x_i, S) \leq \\
        &M(x_i,S) + \mathbb{E}_{\phi}[\text{Gap}_{\phi}(x_i)]. \label{eq: m to M}
    \end{align}
\end{proposition}
We now bound $M$ purely in terms of the properties of the leave-one-out model $h_{\phi'}$. We establish a number of mathematical preliminaries. 

\subsection{Mathematical Preliminaries}

For a pair $z_i\in S$, $z_j\in\Smi$ we use four scalar quantities throughout:
\begin{align}
d_{ij} &:= \|z_i - z_j\|_2, \gamma_j := \|\nabla_{x}\ell(\phi,z_j)\|_{2},\\
\kappa_i &:= \mathbb{E}_{\phi' \sim \mathcal{A}(S^{\setminus i})}[\|\nabla^{2}_{z}\ell(\phi',z_i)\|_{F}^{2}] \label{eq: kappa}\\
C_{ij} &:= \big|\mathbb{E}_{\phi \sim \mathcal{A}(S)}\ell(\phi, z_i) - \mathbb{E}_{\phi' \sim \mathcal{A}(S^{\setminus i})}\ell(\phi', z_j)\big|. \label{eq: gamma-C}
\end{align}
\begin{assumption}($\ell$ has $\rho$-Lipschitz hessians) \label{ass: hessian is lipschitz}
    The Hessian of the loss $\ell$ is Lipschitz in the $z$-variable, i.e for every $\phi \in \Phi$, there exists $\rho$ (possibly dependent on $\phi$) such that: 
    \begin{align}
        \| \nabla^{2}_{z}\ell(\phi, z') - \nabla^{2}_{z}\ell(\phi, z)\|_{2} \leq \rho \|z' - z\|_{2}, \label{eq: hessian is lipschitz} 
    \end{align}
    for every $z,z' \in \mathcal{S}$. 
\end{assumption}

\begin{assumption}(Loss behaves like MSE) \label{ass: bounded gradients assumption}
    We assume that the cost function $c(x,y)$ satisfies 
    \begin{align}
        \|\nabla_yc(x,y)\|^{2}_{2} \leq 2\beta c(x,y). \label{eq: bounded gradients assumption}
    \end{align}
    Note that any $c$ that is $\beta$-Lipschitz in the $y$-variable satisfies \eqref{eq: bounded gradients assumption} \cite{srebro2010optimistic}. For instance, the MSE $c(x,y) = \|x-y\|_{2}^{2}$, or $c(x,y) = \|h(x) - y\|_{2}^{2}$ where $h(x)$ is itself Lipschitz satisfy \eqref{eq: bounded gradients assumption}. 
\end{assumption}

\begin{assumption}(The underlying function is neural network-like) \label{ass: function class is NN}
We assume that all $h^{\phi} = g(\phi_1x, \phi_2)$ where $\phi = (\phi_1, \phi_2)$. Most neural network architectures satisfy this assumption. 
\end{assumption}

Note that the term in \eqref{eq: perturbation bound} is \emph{not quite} $M(z_i,S)$ because the second term in the absolute value involves $z_j$ (which is part of $S^{\setminus i}$) and not $z_i$. 

\begin{assumption}\label{ass: first order optimizer}
    We assume that $\mathcal{A}$ is a first order optimizer in the parameters, i.e there exist $\varepsilon_{\mathrm{grad}}$ and $\varepsilon_{\mathrm{loss}}$ such that
    \begin{align}
        \mathbb{E}_{\phi \sim \mathcal{A}(S)}[\| \nabla_{\phi}\ell(\phi,z_j)\|_{2}] \leq \varepsilon_{\mathrm{grad}},\mathbb{E}_{\phi \sim \mathcal{A}(S)}[ l(\phi,z_j)] \leq \varepsilon_{\mathrm{loss}}, \label{eq: first order optimizer}
    \end{align}
    for every $z_j \in S$. Note that condition \eqref{eq: first order optimizer} is typically met by most deterministic and stochastic optimizers like SGD, Adam, or even full batch gradient descent.  
\end{assumption}

\begin{assumption}\label{ass: bounded weights}
We assume that under $\mathcal{A}(S)$, the weights $\phi$ are almost surely bounded by $B$. 
\end{assumption}

\begin{assumption}\label{ass: single point removal bound}
There exists a constant $C$ independent of $x_i, x_j$ in the leading order (i.e $C = O(1)$ in $x_i, x_j$) such that with high probability in $p^{n}_Z$, for any $z_i \in S, z_j \in S^{\setminus i}$ the following difference is uniformly bounded by $C$: 
{\small
\begin{align}
    \lvert \mathbb{E}_{\phi \sim \mathcal{A}(S)}[\ell(h^{\phi}_{S}, z_i)] - \mathbb{E}_{\phi' \sim \mathcal{A}(S^{\setminus i})}[\ell(h^{\phi'}_{S^{\setminus i}},z_j)]| \leq C. \label{eq: perturbation bound}
\end{align}}
\end{assumption}

 Assumption \ref{ass: first order optimizer} models $\mathcal{A}(S)$ as a first order optimizer that leads to bounded weights upon optimization. In the context of neural networks the weight-boundedness under standard optimization algorithms has been observed in several contexts, from deep linear models \cite{saxe2019mathematical}, feedforward ReLU (or leaky ReLU) networks \cite{du2018algorithmic} to transformers \cite{aubry2024transformer}. 

\subsection{Training-Time Memorization is Controlled by Leave-One-Out Curvature}
\label{sec: main-theorem}

\begin{theorem}
\label{thm: main}
Under Assumption~\ref{ass: hessian is lipschitz}, for every $z_i\in S$ and every $z_j\in\Smi$,
\begin{equation}
|M(z_i;S)| \;\le\; C_{ij} + \gamma_j\,d_{ij} + \tfrac{3}{2}\,\kappa_i\,d_{ij}^2 + \tfrac{2\rho}{3}\,d_{ij}^3 .
\label{eq: main-bound}
\end{equation}
If Assumptions~\ref{ass: bounded gradients assumption}--\ref{ass: bounded weights} hold, then
\begin{equation}
\gamma_j\;\le\;\frac{B\,\varepsilon_{\mathrm{grad}}}{\|z_j\|_2} + \varepsilon_{\mathrm{loss}},
\label{eq: gamma-bound}
\end{equation}
and if Assumption~\ref{ass: single point removal bound} holds as well, then with $\dNN(i) := \min_{j\ne i}d_{ij}$,
\begin{align}
|M(z_i;S)|\;\le\;& C + \dNN(i)^2\left(\tfrac{3}{2}\,\kappa_i\, + \tfrac{2\rho}{3}\,\dNN(i)\right) \nonumber\\
&+ \Big(\frac{B\varepsilon_{\mathrm{grad}}}{\min_j\|z_j\|} + \varepsilon_{\mathrm{loss}}\Big)\dNN(i).
\label{eq: main-bound-nn}
\end{align}
\end{theorem}

\begin{corollary}
    Let Assumptions \ref{ass: hessian is lipschitz}--\ref{ass: single point removal bound} hold and additionally suppose $\varepsilon_{\mathrm{grad}}, \varepsilon_{\mathrm{grad}} \ll \dNN(i)$. Set $C(i):= (C + \mathbb{E}_{\phi'}[\text{Gap}_{\phi'}(x_i)])$. Then: 
    \begin{align}
        &-(C(i) + o(1))- \dNN(i)^2\left(\tfrac{3}{2}\,\kappa_i\, + \tfrac{2\rho}{3}\,\dNN(i)\right)  \\
        &\leq m(x_i, S) \leq \dNN(i)^2\left(\tfrac{3}{2}\,\kappa_i\, + \tfrac{2\rho}{3}\,\dNN(i)\right) + C(i) + o(1).
    \end{align}
\end{corollary}

The proofs are provided in Appendix \ref{app: proofs}. We make the following comments regarding Theorem \ref{thm: main}: 

\begin{enumerate}
\item Under Assumption~\ref{ass: hessian is lipschitz} alone, \eqref{eq: main-bound} is a cubic expansion about an arbitrary second training point $z_j$. Each coefficient can be large and can depend on the query point \(z_i\): the training-loss gap \(C_{ij}\), the leave-one-out input gradient \(\gamma_j\), the leave-one-out curvature \(\kappa_i\), and the Hessian Lipschitz constant \(\rho\).

\item Assumptions~\ref{ass: bounded gradients assumption}--\ref{ass: bounded weights} put a neural network and a first-order method into the bound. When the algorithm overfits, $\varepsilon_{\mathrm{grad}}$ and $\varepsilon_{\mathrm{loss}}$ can be considered to be comparably small against \(\kappa_i\, d_{ij}\), the curvature of the \emph{leave-one-out} loss $\ell(\phi', z_i)$. Curvature thus bounds  counterfactual memorization at the leading order, during overfitting. 

\item Assumption~\ref{ass: single point removal bound} asks for a uniform $C$ that is $O(1)$ in $x$. We can undertake a splitting argument to show that this $C = O(n)$--this is comparable to an analogous bound obtained in \cite{ravikumar2024curvature} (Theorem 5.1) which is also $O(n)$.

\item \cite{ravikumar2024unveiling} obtain $C = O(n)$ using $\beta$-uniform stability and remove the first order terms $\gamma_i$ by comparing \(\mathcal{A}(S)\) with \(\mathcal{A}(S')\) on $\upsilon$-adjacent datasets, where any two points lie at most a distance of $\upsilon$ apart, for small $\upsilon$. \cite{ravikumar2024curvature} remark that a training set must be augmented \emph{post-hoc} to make it $\upsilon$-adjacent. Here Assumption \ref{ass: single point removal bound} replaces the uniform stability assumption. Moreover, we do not require such a condition on the data--instead we restrict the model class to be structured like a neural network in order to use its inductive biases to make the first order terms negligible. Overall, our bound does not require uniform stability, $\upsilon$-adjacency, or bounded loss as is required by Theorem 5.1 in \cite{ravikumar2024curvature}--however, we make more structural but nonetheless realistic assumptions on the model class and the algorithms at hand. 
\end{enumerate}

\subsection{Application to Diffusion Models}

We can make further specifications with additional assumptions for the denoising loss \eqref{eq: denoising loss} and the fact that typically, the hypothesis class $h^{\phi}$ contains linear functions. 

\begin{assumption}
\label{ass: linear function class}
 We assume that for any $M \in \mathbb{R}^{d \times m}$ and $b \in \mathbb{R}^{m}$ there exists some $\phi$ such that $h^{\phi} = Mx + b$.
\end{assumption}

In objective \eqref{eq: denoising loss} we train a denoiser $D_{\phi}: \mathbb{R}^{d} \times \mathbb{R} \to \mathbb{R}^{d}$. Thus we require that the underlying function class contain at least all linear denoisers. This assumption is satisfied by ReLU networks or ReLU U-Nets. Next we state the PL-inequality assumption used for simplifying the analysis of first-order methods in non-convex optimization. Note that neural network loss landscapes have been known to satisfy the local-PL inequality below for the parameter regions of most commonly used neural networks \cite{liu2022loss}. 

\begin{assumption}
\label{ass: local PL}
    Let $\phi$ be the parameters and $\nabla_{\phi}\ell(\phi,S)$ the parameter gradient of the empirical loss. We assume that there exists $\mu > 0$ such that almost surely under $\mathcal{A}(S)$, 
    \begin{align}
        (\ell(\phi, S) - \ell(S^*)) \leq \frac{1}{2\mu} \|\nabla_{\phi} \ell(\phi, S)\|^{2}_{2}
    \end{align}
\end{assumption}

\begin{proposition}
\label{prop: for denoisers C is O(n)}
    Under assumptions \ref{ass: hessian is lipschitz}--\ref{ass: local PL}, if $\ell$ is given by \eqref{eq: denoising loss}, there exists $A(T)$ depending on $f,g,T$ such that the constant $C$ in assumption \ref{ass: single point removal bound} is of the form: 
    \begin{align}
        C \leq \frac{n}{\mu}\varepsilon_{\mathrm{grad}} + nA(T). \label{eq: denoiser C = O(n)}
    \end{align}
\end{proposition}
We provide a proof in Appendix \ref{app: proofs}. Note that in this setting the constant $C$ is comparable to the constant $c_1$ given in Eq. 8 in \cite{ravikumar2024curvature}. 

\subsection{Held-Out Curvature During Training}

Finally, we validate our bound on the DDPMs on Imagenet, to check whether held-out curvature can stand as a proxy for counterfactual memorization. In order to track it during training, we assume bounded ELBO gaps in \eqref{eq: elbo gap} and track the counterfactual loss difference \eqref{eq: train time memorization} instead of the log odds \eqref{eq: sampling time mem}. Since training a diffusion model on $S^{\setminus i}$ for every $i$ is prohibitively expensive, we remove the top 10 most memorized examples for the DDPMs corresponding to $n=128$ and $n=1024$. We then estimate $M(x_i, S)$ for these ten points, with the loss on the $S' := S \setminus \{\text{top 10 memorized points}\}$ with parameters $\phi'$ evaluated on $x_i$ as an estimator for $\ell(h^{\phi'}_{S^{\setminus i}}, x_i)$. We present our results in Figures \ref{fig: ddpm_128_loss_diff_curv} and \ref{fig: ddpm_1024_loss_diff_curv} where we track: 
\begin{align}
    \Delta \ell (\tau) := \ell(h^{\phi}_{S}, z_i) - \ell(h^{\phi'}_{S'}, z_i).
\end{align}
Note that $\Delta \ell$ is the signed version of \eqref{eq: train time memorization}. We observe the following phenomenon with $\Delta \ell$: as training progresses, $\Delta \ell$ drops sharply at the initial phase of training, showing that the full model attempts to fit at the seen example $x_i$. Then the loss stabilizes as both models enter their respective memorization phases. As the $f_{mem}$ approaches $100\%$ $\Delta(\tau)$ starts dropping quickly again. Throughout this whole process, however, the curvature of the held out models at the held out samples roughly corresponds to the drops and stability of $\Delta(\tau)$, thus suggesting that our bound \eqref{eq: main-bound} is a meaningful proxy for the counterfactual loss. At the very tail of training, the curvature for the held-out models \emph{rises}, reflecting the local minimum phenomenon seen in Figures \ref{fig: celebA} and \ref{fig: imagenet}. However, $\Delta(\tau)$ keeps falling, suggested that the held out point's influence on $\Delta(\tau)$ via curvature breaks down in the highly overfitted regime. 

\begin{figure}[h]
    \centering
    \includegraphics[width=0.96\linewidth]{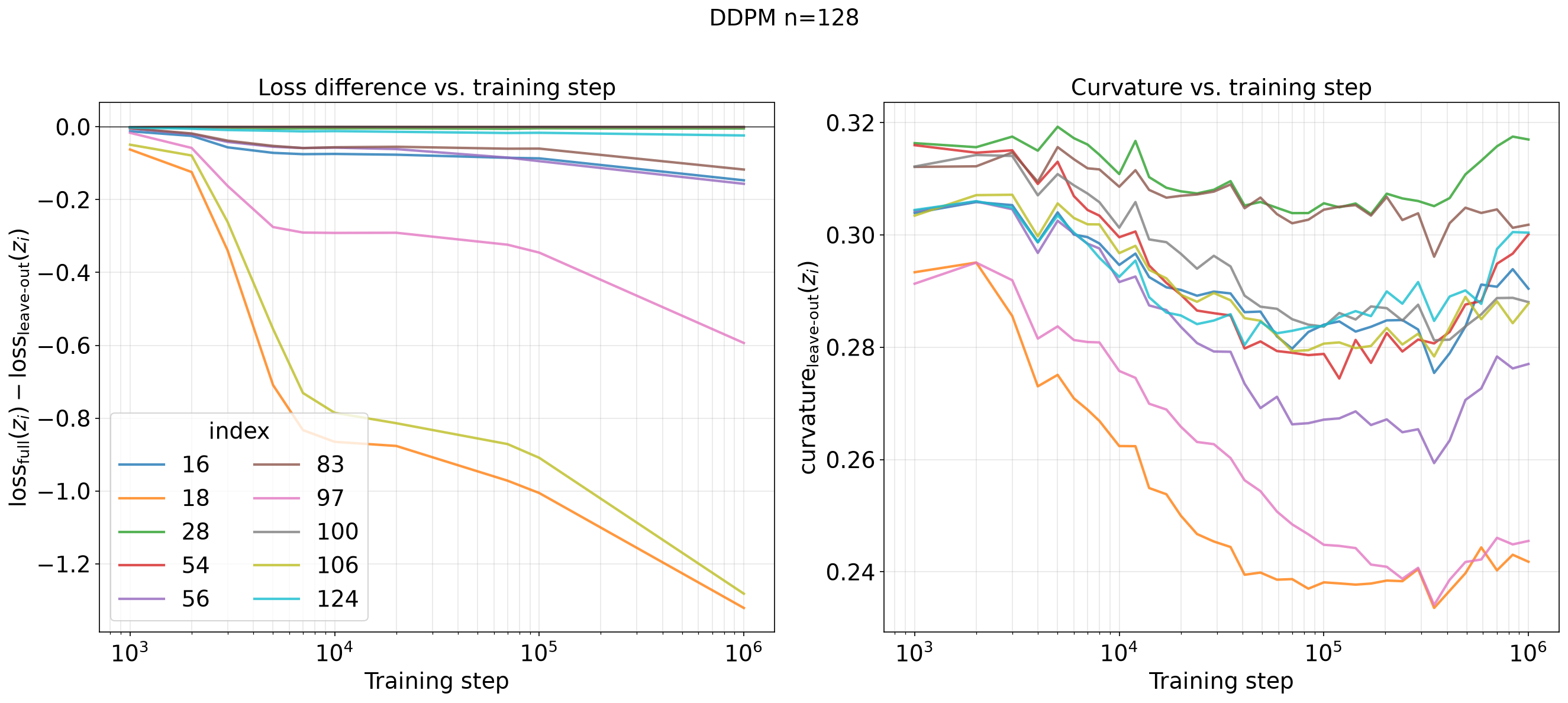}
    \caption{\textbf{Loss difference and curvature for Imagenet (n=128) DDPM noise prediction - leave-10-out training}: (left) difference between per sample loss of full model and leave-10-out model vs. training step for each of the 10 left-out training samples, (right) per sample curvature of the leave-10-out model vs. training step for each of the 10 left-out training samples.}
    \label{fig: ddpm_128_loss_diff_curv}
\end{figure}

\begin{figure}[h]
    \centering
    \includegraphics[width=0.96\linewidth]{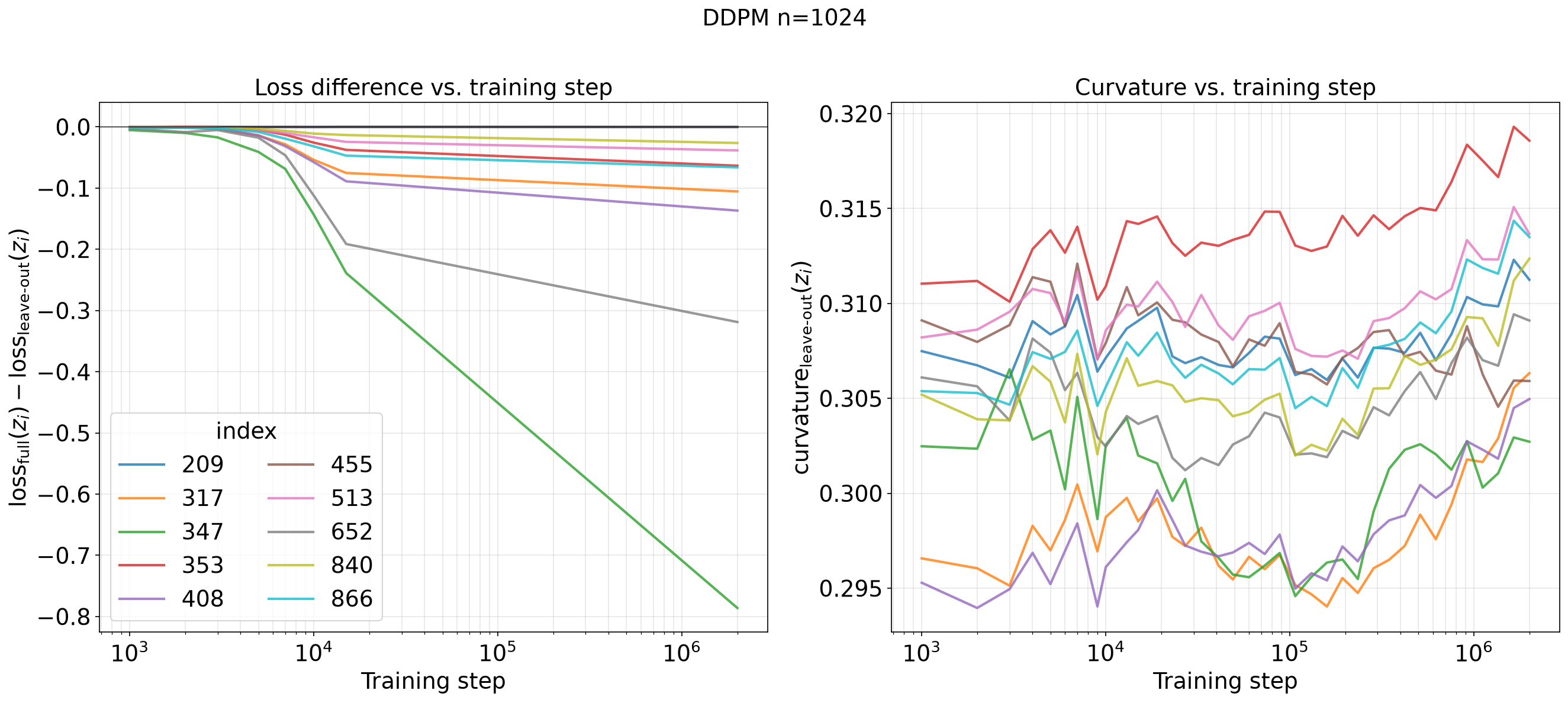}
    \caption{\textbf{Loss difference and curvature for Imagenet (n=1024) DDPM noise prediction - leave-10-out-training}: (left) difference between per sample loss of full model and leave-10-out model vs. training step for each of the 10 left-out training samples, (right) per sample curvature of the leave-10-out model vs. training step for each of the 10 left-out training samples.}
    \label{fig: ddpm_1024_loss_diff_curv}
\end{figure}

\section{Conclusion}
\label{sec: conclusion}

In this paper, we introduced the input-loss curvature as a sampling-free proxy for $f_{\mathrm{mem}}$. Averaged over the training set, it tracks $f_{\mathrm{mem}}$ on ImageNet and on CelebA as training proceeds, albeit it is not sufficient to detect the onset or prevalence of memorization.  Instead, we find that the step at which the validation curvature separates from the training curvature marks the transition into a memorizing model. However, over individual points, the input loss curvature cannot reflect the copy count of individual samples. For that question, we use a counterfactual definition, the change in loss and in log-likelihood at $x$ when $x$ is removed from the training set, and we prove that this quantity is influenced, to leading order in the displacement of $x$, by the input-loss curvature of the leave-one-out model. 

We enumerate the limitations of our findings below: 
\begin{enumerate}
    \item We provide narrow evidence ranging over only two datasets, ImageNet and CelebA. Moreover, a comparison between the various training losses used to train diffusion models merits further investigation that is not performed at length here. 
    \item We propose using the local maximum of the relative train--validation curvature error as a stopping criterion. However, our experiments show that this local maximum indicates a period of already high memorization. An improved statistic would be needed to locate the onset of memorization. For small datasets, Figure \ref{fig: imagenet} shows that the norm of the loss gradient averaged over the validation data might prove useful, but its role and behaviour for larger datasets is not well-understood.
    \item The bound is a plausibility argument and depends on a list of assumptions on the loss, the model, and the optimizer. It is also unsettled whether the curvature term dominates the ELBO gap at $x$, which is the comparison that decides whether curvature is the leading contribution to \eqref{eq: sampling time mem}.
\end{enumerate}

\section*{Acknowledgment}
This work was supported in part by NSF DMS 2502561.

\bibliographystyle{alpha}
\bibliography{iclr2027_conference}

\clearpage
\appendix
\thispagestyle{empty}

\onecolumn

\section{Proofs for Section \ref{sec: sample-level mem}}
\label{app: proofs}

\begin{lemma}[Lemma 1 from \cite{nesterov2006cubic}] \label{lem: cubic envelope lemma}
    Suppose $\ell(\phi, z)$ satsifies Assumption \ref{ass: hessian is lipschitz} of $\rho$-Lipschitz Hessians. Then for any $z,z' \in \mathbb{R}^{d}$
    \begin{align}
        -\frac{\rho}{6}\|z'-z\|_{2}^{3} \leq \ell(\phi, z') - \ell(\phi, z) - \langle \nabla \ell(\phi, z), z-z'\rangle - (z'-z)^{\top}\nabla^{2}_{z}\ell(\phi, z) (z'-z) \leq \frac{\rho}{6}\|z'-z\|_{2}^{3}.
    \end{align}
\end{lemma}

Using Assumption \eqref{ass: hessian is lipschitz} we can also obtain the following bound on the gradients of $\ell$: 

\begin{lemma}\label{lem: cubic envelope on gradient}
        Suppose $\ell(\phi, z)$ satisfies Assumption \ref{ass: hessian is lipschitz} of $\rho$-Lipschitz Hessians. Then for any $z,z' \in \mathbb{R}^{d}$
        \begin{align}
            \| \nabla_{z}\ell(\phi, z') - \nabla_{z}\ell(\phi, z)\|_{2} \leq \|\nabla^{2}_{z}\ell(\phi, z)\|_{2 \to 2}\|z'-z\|_{2} + \frac{\rho}{2}\|z'-z\|_{2}^{2}.  
        \end{align}
        Here $\|\nabla^{2}_{z}\ell(\phi, z)\|_{2 \to 2}$ is the operator norm of the Hessian. 
\end{lemma}

Following an argument of Ma \& Ying \cite{ma2021linear} we use the multiplicative nature of neural networks to prove a bound on the input gradients in terms of the parameter gradients. 

\begin{lemma}\label{lem: input to parameters conversion lemma}
    Suppose $\ell(\phi, z) = g(\phi_1z, \phi_2, z)$ and there exists a constant $C_2$ such that $\| \nabla_c g\| \leq C_2 g$. Then we have 
    \begin{align}
        \|\nabla_{z} \ell(\phi,z)\|_{2} \leq \frac{\|\phi_1\|}{\|z\|}\|\nabla_{\phi_1}\ell(\phi,z)\| + C_2 \ell(\phi, z). \label{eq: input to parameter gradients bound}
    \end{align}
\end{lemma}

\begin{proof}
    We compute the gradient with respect to the inputs: 
    \begin{align}
        \nabla_z \ell(\phi, z)= \nabla_z g(\phi_1z, \phi_2, z) = \phi_1^{\top}\nabla_ag(\phi_1 z, \phi_2, z) + \nabla_cg(\phi_1z,\phi_2,z), \\
        \nabla_{\phi_1}\ell(\phi, z) = \nabla_ag(\phi_1z, \phi_2, z)z^{\top}. 
    \end{align}
    Multiplying the first equation on the right by $z^{\top}$ we get 
    \begin{align}
        \nabla_z \ell(\phi, z)z^{\top} = \phi_1^{\top}\nabla_ag(\phi_1 z, \phi_2, z)z^{\top} + \nabla_cg(\phi_1z,\phi_2,z)z^{\top} \\
    \iff \nabla_z \ell(\phi, z)z^{\top} = \phi_1^{\top}\nabla_{\phi_1}\ell(\phi,z) + \nabla_cg(\phi_1z,\phi_2,z)z^{\top} \\
    \iff \nabla_z \ell(\phi, z) = \frac{\phi_1^{\top}\nabla_{\phi_1}\ell(\phi,z)z}{\|z\|_{2}^{2}} + \nabla_cg(\phi_1z,\phi_2,z).
    \end{align}
    Now, taking norms and using our assumption we have: 
    \begin{align}
        \|\nabla_{z} \ell(\phi,z)\|_{2} \leq \frac{\|\phi_1\|}{\|z\|}\|\nabla_{\phi_1}\ell(\phi,z)\| + C_2 \ell(\phi, z). 
    \end{align}\end{proof}
Before going further, we remark on the significance of Lemma \ref{lem: input to parameters conversion lemma}. In machine learning, we optimize the loss with respect to the parameters in the model $\phi$. However, since we are interested in the effect of the inclusion of a single training sample, we are more interested in the behaviour of the loss with respect to its \emph{inputs} $z$. Lemma \ref{lem: input to parameters conversion lemma} gives us a way to relate these two, specifically relating the sizes of the input and parameter gradients of the loss. In particular, it says that the input gradients can be (roughly) controlled by the parameter gradients, which we explicitly control via optimization. 

Using the cubic envelope from Lemma \eqref{lem: cubic envelope lemma} we may bound the memorization term $M(z_i, S)$ from \eqref{eq: memorization bounded above by cubic envelope}. 

\begin{lemma} \label{lem: bound on mem}
    Let $S$ be a finite dataset and let $S' \subseteq S$ and let $z_i \in S \setminus S'$. Then under assumptions \eqref{ass: hessian is lipschitz} and \eqref{ass: single point removal bound}, for any $z_j \neq z_i \in S'$ we have: 
    \begin{align}
        \mathbb{E}_{\phi}[\ell(h^{\phi}_{S}, z_i)] - \mathbb{E}_{\phi'}[\ell(h^{\phi'}_{S'}, z_j)] &\leq C + \lvert\langle \mathbb{E}_{\phi \sim \mathcal{A}(S')}[\nabla_z\ell(h^{\phi'}_{S'}, z_i)], z_j - z_i\rangle \rvert \notag \\
        &+\lvert (z_j - z_i)^{\top}\mathbb{E}_{\phi \sim \mathcal{A}(S')}[\nabla^{2}_{z}\ell(h^{\phi'}_{S'}, z_i)](z_j-z_i)\rvert + \frac{\rho}{6}\|z_j-z_i\|^{3}_{2}. \label{eq: memorization bounded above by cubic envelope}
    \end{align}
\end{lemma}

\begin{proof}
    Let $\phi' \sim \mathcal{A}(S')$. For $z_j, z_i$, we have, using Lemma \ref{lem: cubic envelope lemma} that: 
    \begin{align}
        -\ell(h^{\phi'}_{S'},z_j) \geq - \ell(h^{\phi'}_{S'}, z_i) - \langle \nabla_z\ell(h^{\phi'}_{S'}, z_i), z_j - z_i\rangle - (z_j - z_i)^{\top}\nabla^{2}_{z}\ell(h^{\phi'}_{S'}, z_i)(z_j-z_i) - \frac{\rho}{6}\|z_j-z_i\|^{3}_{2}, \\
        \ell(h^{\phi'}_{S'},z_j) \leq \ell(h^{\phi'}_{S'}, z_i) + \langle \nabla_z\ell(h^{\phi'}_{S'}, z_i), z_j - z_i\rangle + (z_j - z_i)^{\top}\nabla^{2}_{z}\ell(h^{\phi'}_{S'}, z_i)(z_j-z_i) + \frac{\rho}{6}\|z_j-z_i\|^{3}_{2}. 
    \end{align}
    Using the monotonicity of expectation we therefore get: 
    \begin{align}
        -\mathbb{E}_{\phi \sim \mathcal{A}(S')}[\ell(h^{\phi'}_{S'},z_j)] &\geq - \mathbb{E}_{\phi \sim \mathcal{A}(S')}[\ell(h^{\phi'}_{S'}, z_i)] - \langle \mathbb{E}_{\phi \sim \mathcal{A}(S')}[\nabla_z\ell(h^{\phi'}_{S'}, z_i)], z_j - z_i\rangle \\
        &- (z_j - z_i)^{\top}\mathbb{E}_{\phi \sim \mathcal{A}(S')}[\nabla^{2}_{z}\ell(h^{\phi'}_{S'}, z_i)](z_j-z_i) - \frac{\rho}{6}\|z_j-z_i\|^{3}_{2}, \\
        \mathbb{E}_{\phi \sim \mathcal{A}(S')}[\ell(h^{\phi'}_{S'},z_j)] &\leq \mathbb{E}_{\phi \sim \mathcal{A}(S')}[\ell(h^{\phi'}_{S'}, z_i)] + \langle \mathbb{E}_{\phi \sim \mathcal{A}(S')}[\nabla_z\ell(h^{\phi'}_{S'}, z_i)], z_j - z_i\rangle \\
        &+ (z_j - z_i)^{\top}\mathbb{E}_{\phi \sim \mathcal{A}(S')}[\nabla^{2}_{z}\ell(h^{\phi'}_{S'}, z_i)](z_j-z_i) + \frac{\rho}{6}\|z_j-z_i\|^{3}_{2}.
    \end{align}
    Plugging these into the two-sided bound implied by Assumption \ref{ass: single point removal bound} we get: 
    \begin{align}
        C &\geq \mathbb{E}_{\phi}[\ell(h^{\phi}_{S}, z_i)] - \mathbb{E}_{\phi'}[\ell(h^{\phi'}_{S'}, z_j)] \\
        &\geq \mathbb{E}_{\phi}[\ell(h^{\phi}_{S}, z_i)] - \mathbb{E}_{\phi \sim \mathcal{A}(S')}[\ell(h^{\phi'}_{S'}, z_i)] - \langle \mathbb{E}_{\phi \sim \mathcal{A}(S')}[\nabla_z\ell(h^{\phi'}_{S'}, z_i)], z_j - z_i\rangle \\
        &- (z_j - z_i)^{\top}\mathbb{E}_{\phi \sim \mathcal{A}(S')}[\nabla^{2}_{z}\ell(h^{\phi'}_{S'}, z_i)](z_j-z_i) - \frac{\rho}{6}\|z_j-z_i\|^{3}_{2}.
    \end{align}
 By rearranging, we have 
    \begin{align}
        \mathbb{E}_{\phi}[\ell(h^{\phi}_{S}, z_i)] - \mathbb{E}_{\phi'}[\ell(h^{\phi'}_{S'}, z_j)] &\leq C + \langle \mathbb{E}_{\phi \sim \mathcal{A}(S')}[\nabla_z\ell(h^{\phi'}_{S'}, z_i)], z_j - z_i\rangle \\
        &+(z_j - z_i)^{\top}\mathbb{E}_{\phi \sim \mathcal{A}(S')}[\nabla^{2}_{z}\ell(h^{\phi'}_{S'}, z_i)](z_j-z_i) + \frac{\rho}{6}\|z_j-z_i\|^{3}_{2}.
    \end{align}
    Using the same tactic, we may prove a lower bound for $\mathbb{E}_{\phi}[\ell(h^{\phi}_{S}, z_i)] - \mathbb{E}_{\phi'}[\ell(h^{\phi'}_{S'}, z_j)]$ using the negative of the term above. This yields a bound for the absolute value, and recalling the definition of $M$ yields the bound \eqref{eq: memorization bounded above by cubic envelope}. 
\end{proof}

We can now and prove our main result \ref{thm: main}. 

\begin{proof}
    First, we construct the high probability event under which the bound should hold. Set the threshold in assumption \ref{ass: first order optimizer} to be $o(n^{-1/2})$ where $n$ is the size of the dataset. Then fix a threshold $1-\delta(n)$ such that both events under which assumptions \ref{ass: single point removal bound} and \ref{ass: bounded weights} hold. In this high probability event we apply Lemma \ref{lem: bound on mem} to the set $S' = S^{\setminus i} := S \setminus \{z_i\}$ to obtain \eqref{eq: memorization bounded above by cubic envelope}: 
\begin{align}
    M(z_i, S) &\leq C + \underbrace{{\lvert\langle \mathbb{E}_{\phi \sim \mathcal{A}(S^{\setminus i})}[\nabla_z\ell(h^{\phi'}_{S^{\setminus i}}, z_i)], z_j - z_i\rangle \rvert}}_{A} \notag \\
    &+ \underbrace{{\lvert (z_j - z_i)^{\top}\mathbb{E}_{\phi \sim \mathcal{A}(S^{\setminus i})}[\nabla^{2}_{z}\ell(h^{\phi'}_{S^{\setminus i}}, z_i)](z_j-z_i)\rvert + \frac{\rho}{6}\|z_j-z_i\|^{3}_{2}}}_{B}. \notag \tag{\ref{eq: memorization bounded above by cubic envelope}}
\end{align}

    To bound {\color{blue} (A)}, we use Cauchy-Schwarz and Jensen's inequality: 
    \begin{align}
        \lvert\langle \mathbb{E}_{\phi \sim \mathcal{A}(S^{\setminus i})}[\nabla_z\ell(h^{\phi'}_{S^{\setminus i}}, z_i)], z_j - z_i\rangle \rvert \leq \mathbb{E}_{\phi \sim \mathcal{A}(S^{\setminus i})}[\|\nabla_z\ell(h^{\phi'}_{S^{\setminus i}}, z_i) \|_{2}] \|z_j-z_i\|_{2}. \label{eq: bounding the first order term in cubic envelope}
    \end{align}
    We now apply Lemma \ref{lem: cubic envelope on gradient} on $\nabla_z\ell(h^{\phi'}_{S^{\setminus i}}, z_i) \|_{2}$ and take expectations to obtain: 
    \begin{align}
        \mathbb{E}_{\phi \sim \mathcal{A}(S^{\setminus i})}[\|\nabla_z\ell(h^{\phi'}_{S^{\setminus i}}, z_i) \|_{2}] \|z_j-z_i\|_{2} &\leq \mathbb{E}_{\phi \sim \mathcal{A}(S^{\setminus i})}[\|\nabla_z\ell(h^{\phi'}_{S^{\setminus i}}, z_j) \|_{2}]\|z_j-z_i\|_{2} \label{eq: bound containing input gradient}\\
        &+ \mathbb{E}_{\phi' \sim \mathcal{A}(S^{\setminus i})}[\|\nabla^{2}_{z}\ell(h^{\phi'}_{S^{\setminus i}},z_i)\|_{2\to 2}]\|z_j - z_i\|_{2}^{2} + \frac{\rho}{2}\|z_j-z_i\|_{2}^{3}. \label{eq: cauchy schwarz bound on first order stuff}
    \end{align}
    We now apply Lemma \ref{lem: input to parameters conversion lemma} to \eqref{eq: bound containing input gradient} and thus get: 
    \begin{align}
        \mathbb{E}_{\phi \sim \mathcal{A}(S^{\setminus i})}[\|\nabla_z\ell(h^{\phi'}_{S^{\setminus i}}, z_j) \|_{2}]\|z_j-z_i\|_{2} &\leq \Bigg(\mathbb{E}_{\phi' \sim \mathcal{A}(S^{\setminus i})}\Bigg[\frac{\|\phi'_1\|}{\|z\|}\|\nabla_{\phi'_1}\ell(h^{\phi'}_{S^{\setminus i}},z_j)\| + 2\beta \ell(h^{\phi'}, z_j)\Bigg]\Bigg) \|z_j-z_i\|_{2} \\
        &= \mathbb{E}_{\phi' \sim \mathcal{A}(S^{\setminus i}})\left(\frac{B}{\min_{k} \|z_k\|}\varepsilon_{\mathrm{grad}} + 2\beta \varepsilon_{\mathrm{loss}}\right)\|z_j - z_i\|_{2}\label{eq: bound on first term in first order term}
    \end{align}
    The last equality follows due to assumption \ref{ass: first order optimizer} and \ref{ass: bounded weights} which results in both the loss and the gradients of the loss of a trained model at a training sample to be small (compared to the parameters and the nearest neighbor distance $\|z_j-z_i\|$). Putting \eqref{eq: bound on first term in first order term} with \eqref{eq: bound containing input gradient} we get: 
    \begin{align}
        \lvert\langle \mathbb{E}_{\phi \sim \mathcal{A}(S^{\setminus i})}[\nabla_z\ell(h^{\phi'}_{S^{\setminus i}}, z_i)], z_j - z_i\rangle \rvert &\leq \mathbb{E}_{\phi' \sim \mathcal{A}(S^{\setminus i})}[\|\nabla^{2}_{z}\ell(h^{\phi'}_{S^{\setminus i}},z_i)\|_{2\to 2}]\|z_j - z_i\|_{2}^{2} \\
        &+ \mathbb{E}_{\phi' \sim \mathcal{A}(S^{\setminus i}})\left(\frac{B}{\min_{k} \|z_k\|}\varepsilon_{\mathrm{grad}} + 2\beta \varepsilon_{\mathrm{loss}}\right)\|z_j - z_i\|_{2} + O(\|z_j-z_i\|_{2}^{3}). \label{eq: due to optimization, first order info is bounded above by hessian} 
    \end{align}
    To bound {\color{orange}(B)} we use an operator norm upper bound along with Jensen's inequality to bound the quadratic term in \eqref{eq: memorization bounded above by cubic envelope} also by the operator norm of the hessian appearing in \eqref{eq: due to optimization, first order info is bounded above by hessian}: 
    \begin{align}
    \lvert (z_j - z_i)^{\top}\mathbb{E}_{\phi' \sim \mathcal{A}(S^{\setminus i})}[\nabla^{2}_{z}\ell(h^{\phi'}_{S^{\setminus i}}, z_i)](z_j-z_i)\rvert  &= \lvert (z_j - z_i)^{\top}\mathbb{E}_{\phi' \sim \mathcal{A}(S^{\setminus i})}[\nabla^{2}_{z}\ell(h^{\phi'}_{S^{\setminus i}}, z_i)(z_j-z_i)]\rvert\\
    &\leq \| \mathbb{E}_{\phi \sim \mathcal{A}(S^{\setminus i})}[\nabla^{2}_{z}\ell(h^{\phi'}_{S^{\setminus i}}, z_i)(z_j-z_i)] \| \|z_j - z_i\| \\
    &\leq \mathbb{E}_{\phi' \sim \mathcal{A}(S^{\setminus i})}[\|\nabla^{2}_{z}\ell(h^{\phi'}_{S^{\setminus i}},z_i)\|_{2\to 2}]\|z_j - z_i\|_{2}^{2} \\
    & \leq \mathbb{E}_{\phi' \sim \mathcal{A}(S^{\setminus i})}[\|\nabla^{2}_{z}\ell(h^{\phi'}_{S^{\setminus i}},z_i)\|^{2}_{F}]\|z_j - z_i\|_{2}^{2}\label{eq: upper bound on quadratic term in cubic envelope}
    \end{align}
    Plugging \eqref{eq: due to optimization, first order info is bounded above by hessian} and \eqref{eq: upper bound on quadratic term in cubic envelope} yields \eqref{eq: main-bound}.
\end{proof}

\subsection{Proof of Proposition \ref{prop: for denoisers C is O(n)}}

For brevity we set $z = z_i, z':= z_j$, and $S' = S^{\setminus i}$. First, since the function class $h^{\phi}$ contains linear denoisers, the performance of $h^{\phi}$ under $\mathcal{A}(S)$ is \emph{better} than any linear denoiser. However, the best linear denoiser attains an MSE dependent on $T, f, g$ \cite{weitzner2025diffusion} that we term $A(T)$. Thus $\ell^*(S) \leq A(T)$. We may now decompose the difference in assumption \ref{ass: single point removal bound} as follows by introducing the rest of the dataset $S, S'$: 

\begin{align*}
    \ell(h^{\phi}, z) - \ell(h^{\phi'}, z') 
    &= |S|\left(\ell(h^{\phi}, S) - \ell^*(S)\right) - |S'|\left(\ell(h^{\phi'}, S') - \ell^*(S')\right) + |S'|\left(\ell^*(S) - \ell^*(S')\right) \notag \\
    &+ |S \setminus S'|\ell^*(S) + \sum_{z_k \in S', z_k \neq z'} \ell(h^{\phi'}, z_k) - \sum_{z_k \in S, z_k \neq z} \ell(h^{\phi}, z_k) \notag \\
    &\leq {|S|\left(\ell(h^{\phi}, S) - \ell^*(S)\right)} 
    - {|S'|\left(\ell(h^{\phi'}, S') - \ell^*(S')\right)} 
    + |S'|{\left(\ell^*(S) - \ell^*(S')\right)} \\
    &+ |S - S'|\ell^*(S). \notag 
\end{align*}

Here the inequality on the third line follows because of the following calculation: 
\begin{align}
    &\sum_{z_k \in S', z_k \neq z'} \ell(h^{\phi'}, z_k) - \sum_{z_k \in S, z_k \neq z} \ell(h^{\phi}, z_k) \notag \\&\leq \sum_{z_k \in S'} \ell(h^{\phi'}, z_k) - \sum_{z_k \in S'} \ell(h^{\phi}, z_k) = |S'|(\ell(h^{\phi'}, S') - \ell(h^{\phi}, S')) \leq 0. \notag
\end{align}
Using the PL-inequality in assumption \ref{ass: local PL}, the bound $\ell^*(S) \leq A(T)$ and the fact that $l^*(S') \geq 0$ we get: 
\begin{align}
    \ell(h^{\phi}, z) - \ell(h^{\phi'}, z')  \leq \frac{n}{\mu}\varepsilon_{\mathrm{grad}} + nA(T). 
\end{align}

\section{Experimental Details}
\label{app: experimental details}
Here we describe our experimental details. Note that we also conducted analogous versions of our experiments for \emph{flow-matching} which we describe below. 

Flow matching \cite{lipman2023flow,liu2023flow} learns a time-dependent velocity field $v_\theta(x,t)$ whose flow transports a noise distribution $p_0$ to the data distribution $p_1$. Given $X_0 \sim p_0$ and $X_1 \sim p_1$, one forms the linear interpolant $X_t = (1-t)X_0 + tX_1$ for $t\in[0,1]$, whose conditional velocity is $\tfrac{d}{dt}X_t = X_1 - X_0$. The field is trained by regressing onto this target,
\begin{align}
  \ell(\phi, x) = \mathbb{E}_{t,\,X_0,\,X_1}\left[\, \big\| v_\theta(X_t, t) - (X_1 - X_0) \big\|_2^2 \, \mid X_0 = x\right],
\end{align}
and the minimizer is the marginal velocity $v^\star(x,t)=\mathbb{E}[X_1 - X_0 \mid X_t = x]$. Samples are then generated by integrating the ODE $\dot{x}_t = v_\theta(x_t,t)$ from $t=0$ to $t=1$.

We conduct numerical experiments on the Imagenet and CelebA datasets. We use ImageNet-1k \cite{deng2009imagenet, russakovsky2015imagenet} and convert to RGB downscaled images of size $d=64\times 64$. For Imagenet, we train on $n\in\{128,256,512,1024,2048,4096,10240\}$. We convert the CelebA \cite{liu2015deep} face images to grayscale and downsample them to size $d=32\times 32$. We use a training set size of $n\in\{128,10240\}$ for CelebA, yielding a memorizing vs. generalizing model. All our score models have a U-Net architecture \cite{ronneberger2015unet} with four resolution levels,  base channel width $W=64$, and channel multipliers 1, 2, 4, and 4 respectively. Self-attention is applied at the two lowest-resolution levels, and each level has two residual blocks in the encoder and three in the decoder. The models are conditioned on diffusion time $t$  through the sinusoidal embedding \cite{vaswani2017attention} in every residual block, so a single network serves all noise levels. We train the network in one of three ways: DDPM \cite{ho2020denoising}with either noise prediction ($\epsilon$-prediction) or clean-sample prediction ($x_0$-prediction), or flow matching \cite{lipman2023flow}. The DDPM noise prediction predicts the injected Gaussian noise at diffusion time $t$ using the Adam optimizer \cite{kingma2014adam} with a fixed batch size of 128. DDPM sample prediction takes the noisy images $x_t$ and time $t$ and predicts the clean image $x_0$ directly using loss $\mathrm{MSE}(\hat{x_0}(x_t,t),x_0)$. DDPM models are sampled using DDIM with 100 steps and stochasticity parameter $\eta = 0$, starting from $x_T \sim \mathcal{N}(0,I)$. For the noise prediction models, the predicted noise $\hat{\epsilon}$ is converted to a clean image estimate $\hat{x_0} = (x_t-\sqrt{(1-\bar{\alpha}_t)}\hat{\epsilon})/\sqrt{\bar{\alpha}_t}$ before stepping to the next time. The sample prediction models use the same DDIM sampler with the network output taken as the clean image estimate $\hat{x_0}$ directly. The flow matching networks predict the velocity $u=x_0-\epsilon$ along the linear path $x_t=(1-\tau)\epsilon+\tau x_0$, where $\epsilon \sim \mathcal{N}(0,I)$ is the noise, $x_0$ is the training image, and $\tau$ is the normalized time. Training minimizes $\mathrm{MSE}(\hat{v}(x_t,t),u)$, and sampling integrates $\tfrac{\mathrm{d}x}{\mathrm{d}\tau}=\hat{v}$ from $\tau=0$ (noise) to $\tau = 1$ (data) with 100 Euler steps.

\section{Generated Samples}
\label{app: generations}
Figures \ref{fig:nn_samples_celeba_ddpm_128}-\ref{fig:nn_samples_imagenet_fm_1024_leave10out} shows generated samples at the pre-memorization, beginning of memorization, and full memorization stages with their nearest training neighbor for DDPM noise prediction and flow matching on CelebA with $n\in\{128,10240\}$ as well as the leave-10-out models for Imagenet with $n\in\{128,1024\}$.

\begin{figure}[h]
    \centering
    \includegraphics[width=0.9\linewidth]{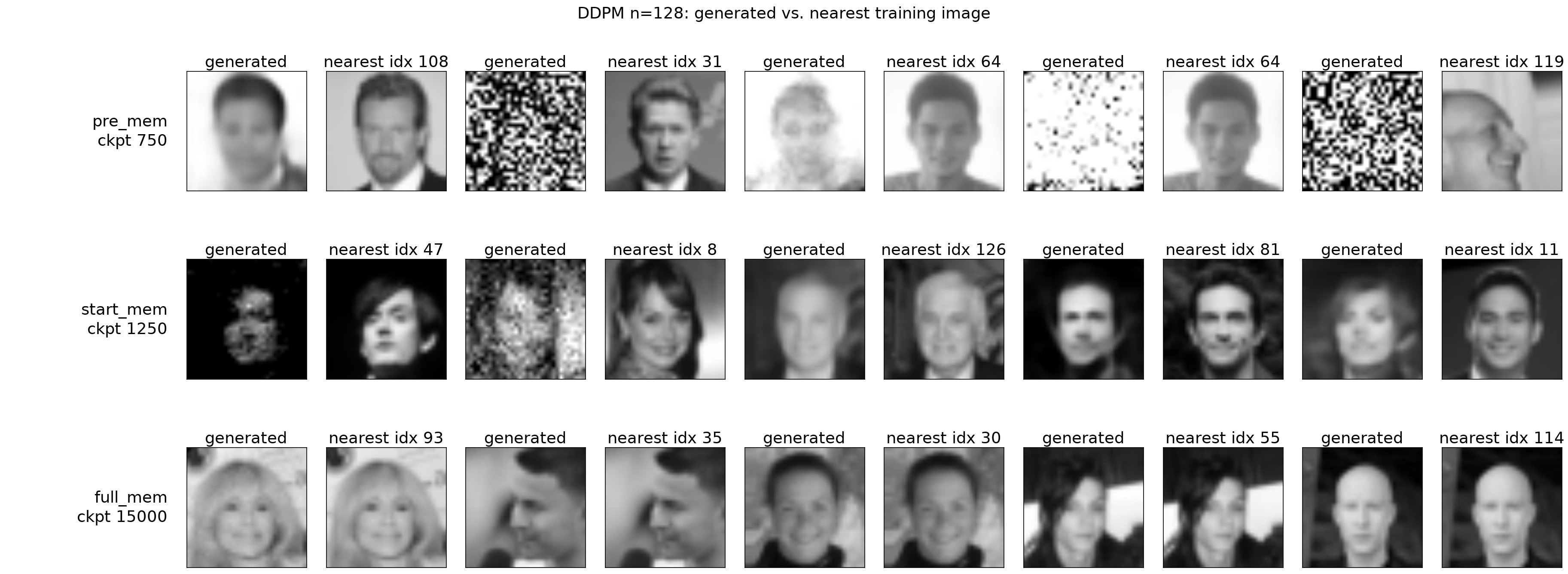}
    \caption{\textbf{Generated samples vs. nearest training image for CelebA (n=128) DDPM}: generated samples at the pre-memorization, start of memorization, and full memorization stages alongside their nearest neighbor in the training set.}
    \label{fig:nn_samples_celeba_ddpm_128}
\end{figure}

\begin{figure}[h]
    \centering
    \includegraphics[width=0.9\linewidth]{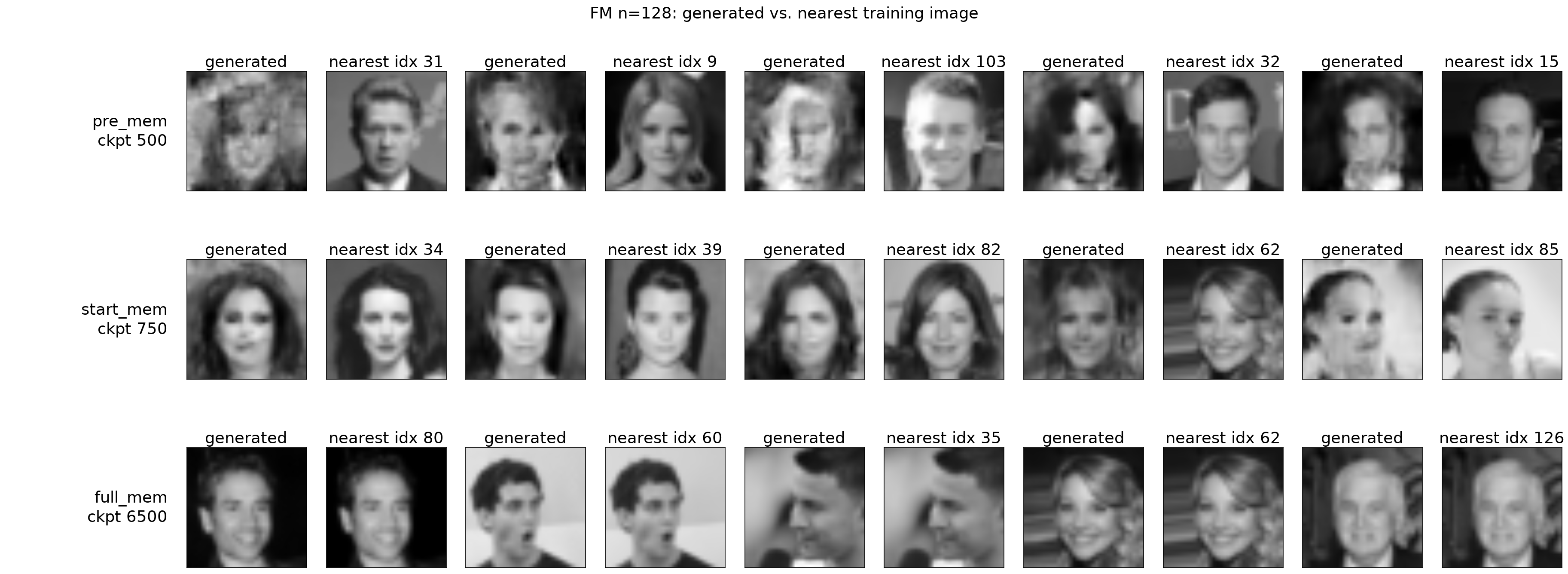}
    \caption{\textbf{Generated samples vs. nearest training image for CelebA (n=128) flow matching}: generated samples at the pre-memorization, start of memorization, and full memorization stages alongside their nearest neighbor in the training set.}
    \label{fig:nn_samples_celeba_fm_128}
\end{figure}

\begin{figure}[h]
    \centering
    \includegraphics[width=0.9\linewidth]{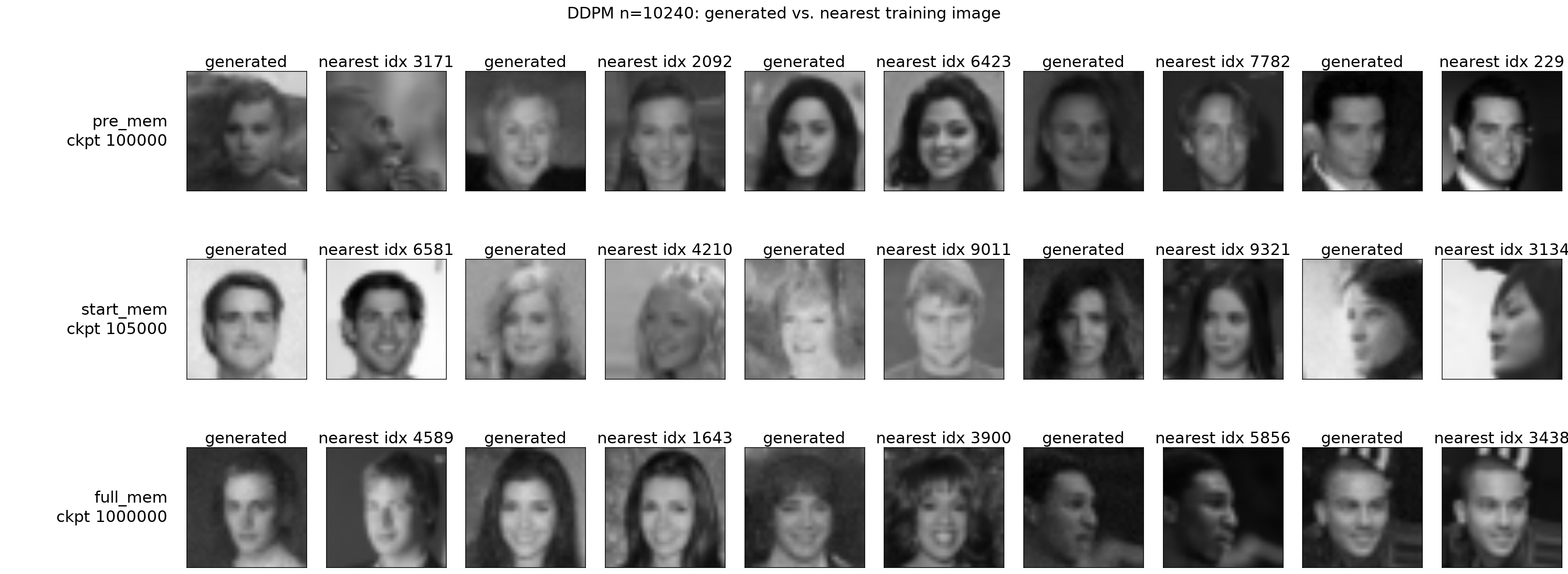}
    \caption{\textbf{Generated samples vs. nearest training image for CelebA (n=10240) DDPM}: generated samples at the pre-memorization, start of memorization, and full memorization stages alongside their nearest neighbor in the training set.}
    \label{fig:nn_samples_celeba_ddpm_10240}
\end{figure}

\begin{figure}[h]
    \centering
    \includegraphics[width=0.9\linewidth]{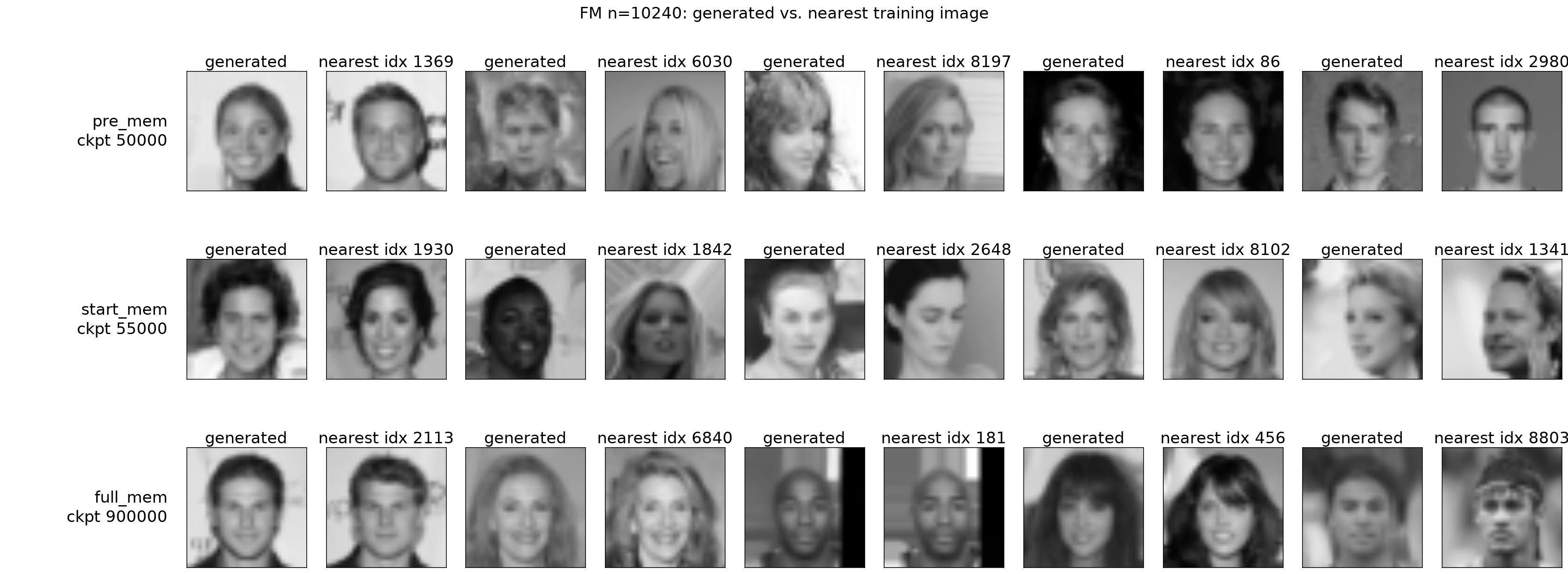}
    \caption{\textbf{Generated samples vs. nearest training image for CelebA (n=10240) flow matching}: generated samples at the pre-memorization, start of memorization, and full memorization stages alongside their nearest neighbor in the training set.}
    \label{fig:nn_samples_celeba_fm_10240}
\end{figure}

\begin{figure}[h]
    \centering
    \includegraphics[width=0.9\linewidth]{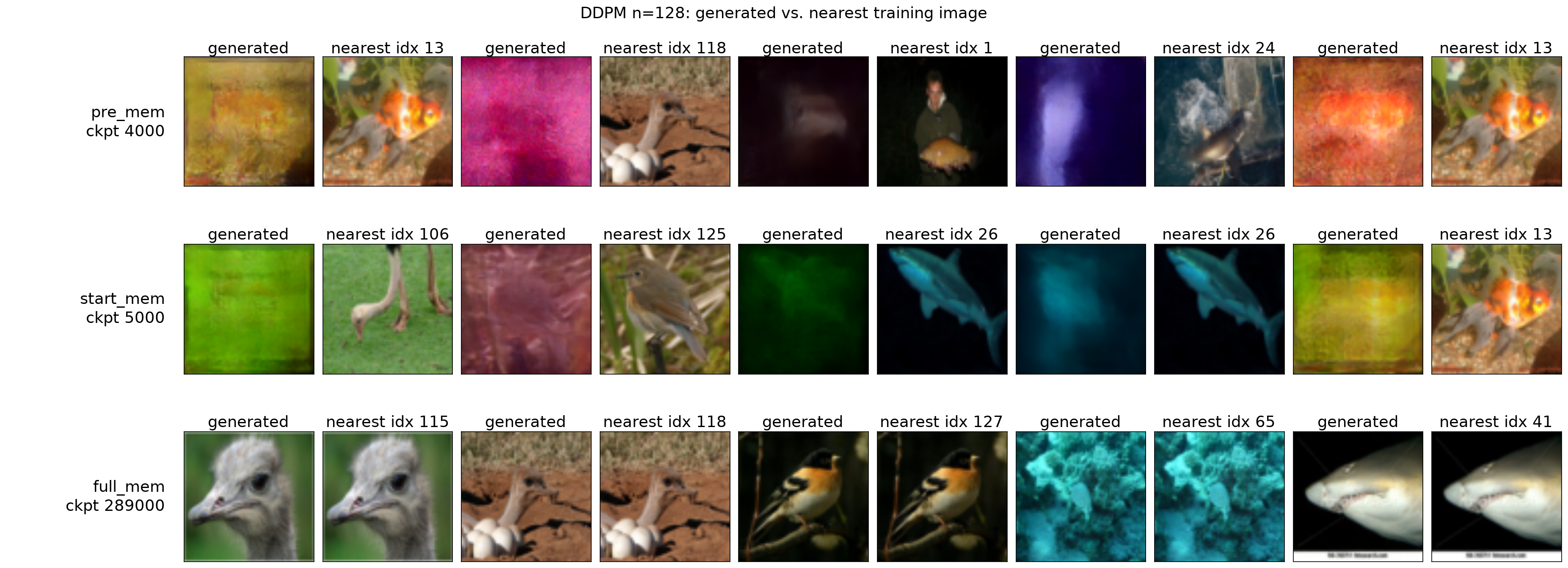}
    \caption{\textbf{Generated samples vs. nearest training image for Imagenet (n=128) DDPM leave-10-out model}: generated samples at the pre-memorization, start of memorization, and full memorization stages alongside their nearest neighbor in the training set.}
    \label{fig:nn_samples_imagenet_ddpm_128_leave10out}
\end{figure}

\begin{figure}[h]
    \centering
    \includegraphics[width=0.9\linewidth]{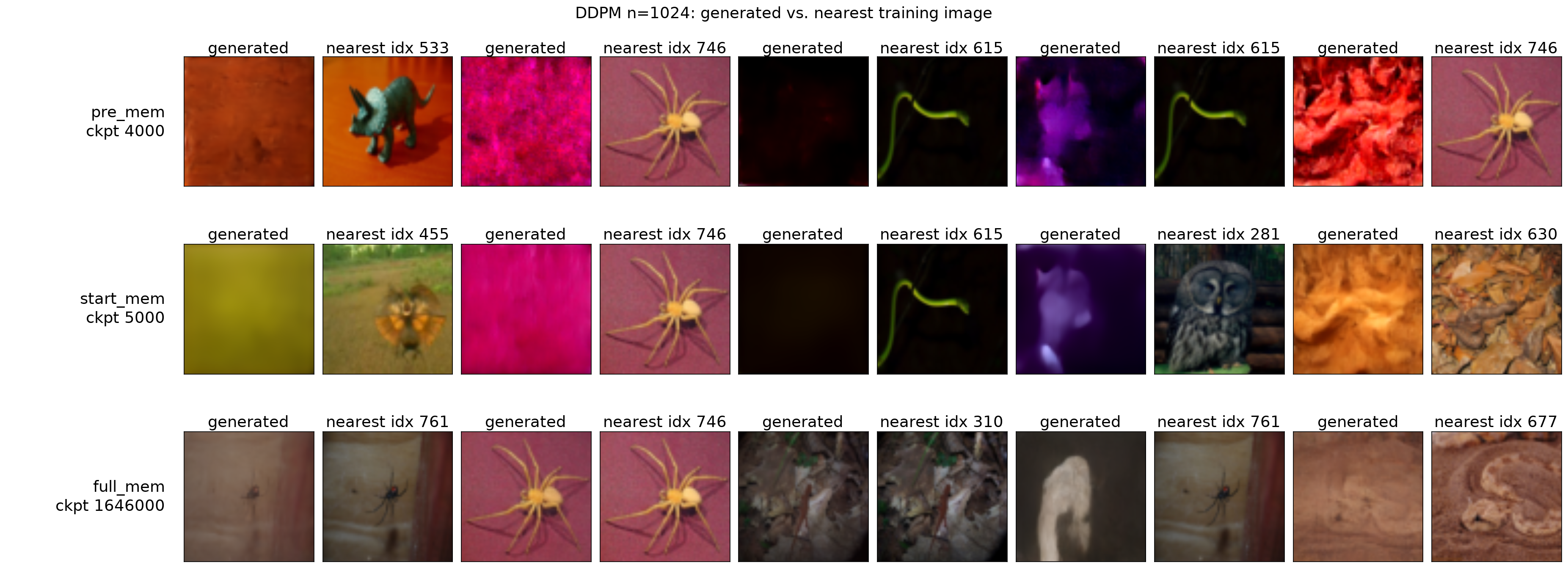}
    \caption{\textbf{Generated samples vs. nearest training image for Imagenet (n=1024) DDPM leave-10-out model}: generated samples at the pre-memorization, start of memorization, and full memorization stages alongside their nearest neighbor in the training set.}
    
    \label{fig:nn_samples_imagenet_ddpm_1024_leave10out}
\end{figure}

\begin{figure}[h]
    \centering
    \includegraphics[width=0.9\linewidth]{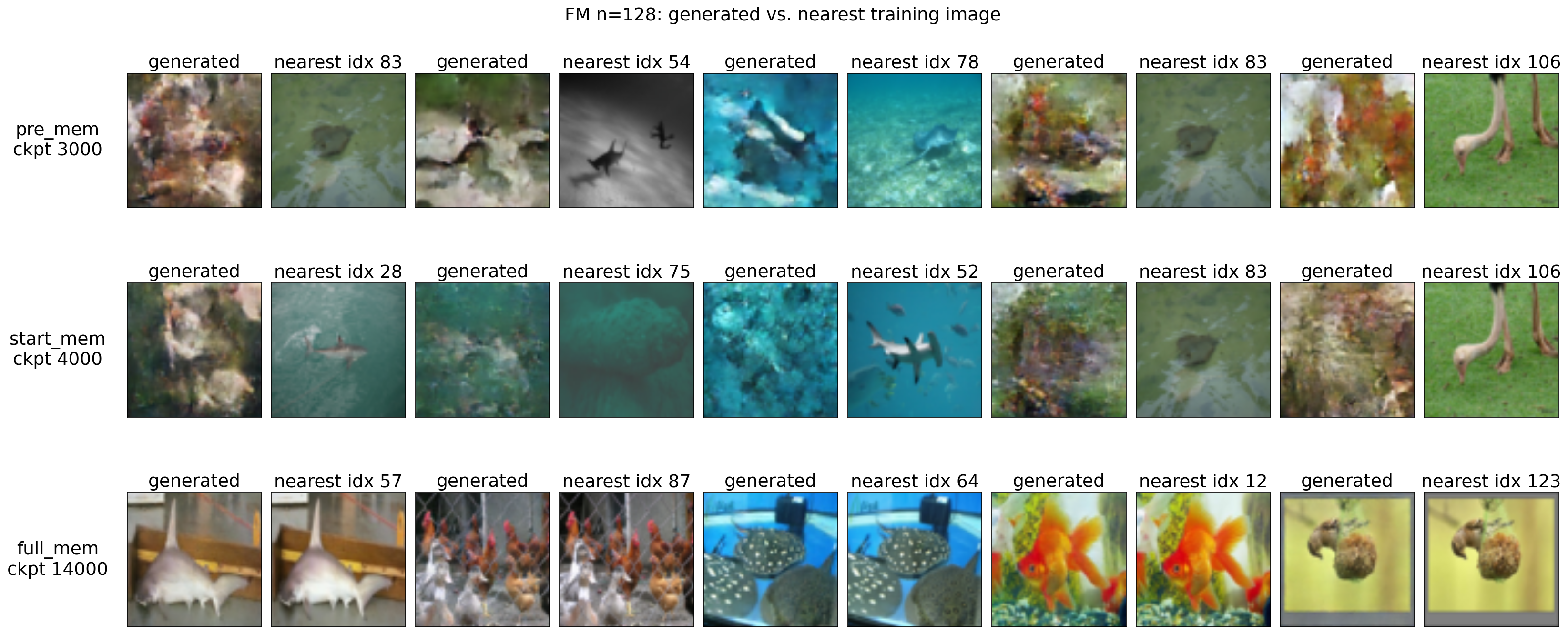}
    \caption{\textbf{Generated samples vs. nearest training image for Imagenet (n=128) flow matching leave-10-out model}: generated samples at the pre-memorization, start of memorization, and full memorization stages alongside their nearest neighbor in the training set.}
    \label{fig:nn_samples_imagenet_fm_128_leave10out}
\end{figure}

\begin{figure}[h]
    \centering
    \includegraphics[width=0.9\linewidth]{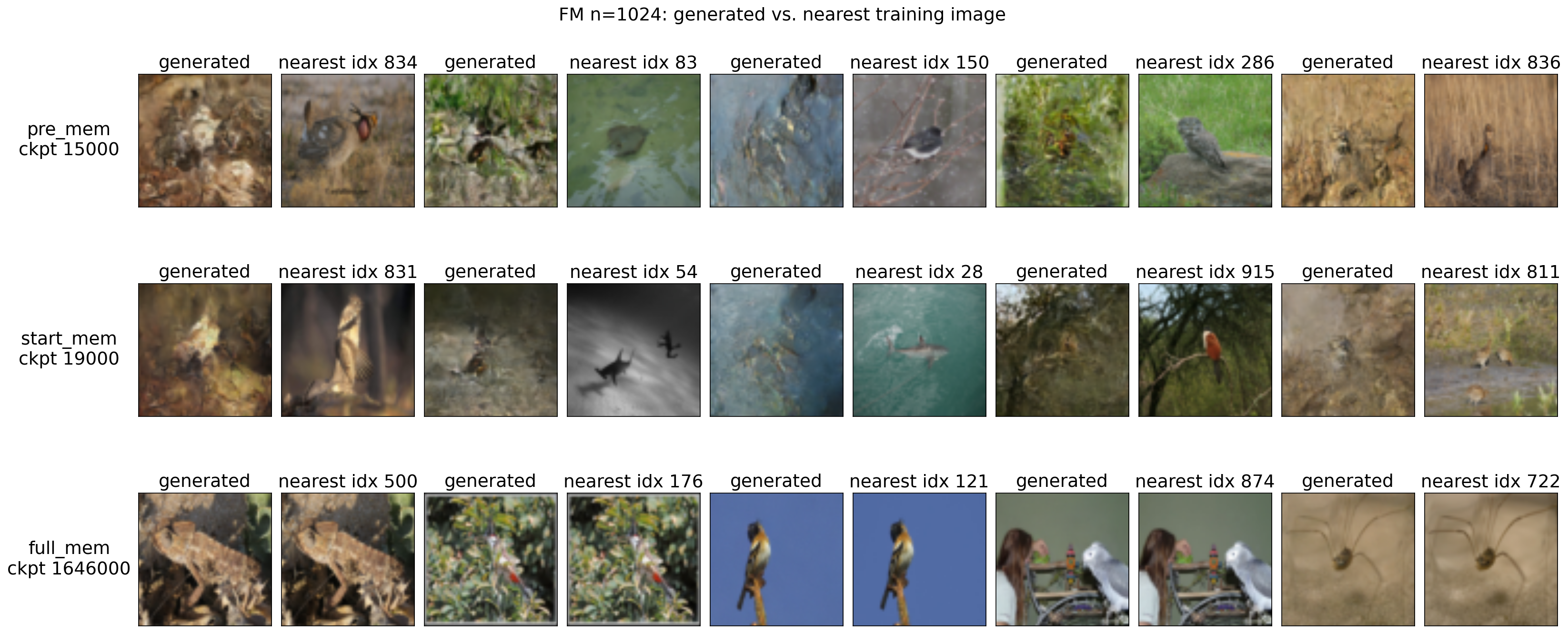}
    \caption{\textbf{Generated samples vs. nearest training image for Imagenet (n=1024) flow matching leave-10-out model}: generated samples at the pre-memorization, start of memorization, and full memorization stages alongside their nearest neighbor in the training set.}
    \label{fig:nn_samples_imagenet_fm_1024_leave10out}
\end{figure}

Figures~\ref{fig:imagenet-stages-n128}--\ref{fig:imagenet-stages-fm-n1024}
show generated samples and their nearest training images for models trained
on all $n$ images: DDPM noise prediction with
$n\in\{128,256,512,1024,2048,4096\}$ and flow matching with
$n\in\{128,1024\}$. Each row contains five generated--training image pairs.
All memorization fractions are measured over $10{,}000$ generations.

For models with sustained memorization, pre\_mem denotes the last evaluated
checkpoint with $f_{\mathrm{mem}}=0$ before onset, and start\_mem denotes
the first checkpoint after which $f_{\mathrm{mem}}$ stays above zero.
Full memorization (full\_mem) denotes the first evaluated checkpoint with
$f_{\mathrm{mem}}\geq95\%$. For DDPM with $n=512$ and $n=1024$, which
never reach this threshold, peak\_mem denotes the highest observed
memorization fraction. For DDPM with $n=4096$, the two rows instead show
a zero-copy checkpoint followed by an isolated copy event.

\begin{figure}[!htbp]
    \centering
    \includegraphics[width=0.9\linewidth]{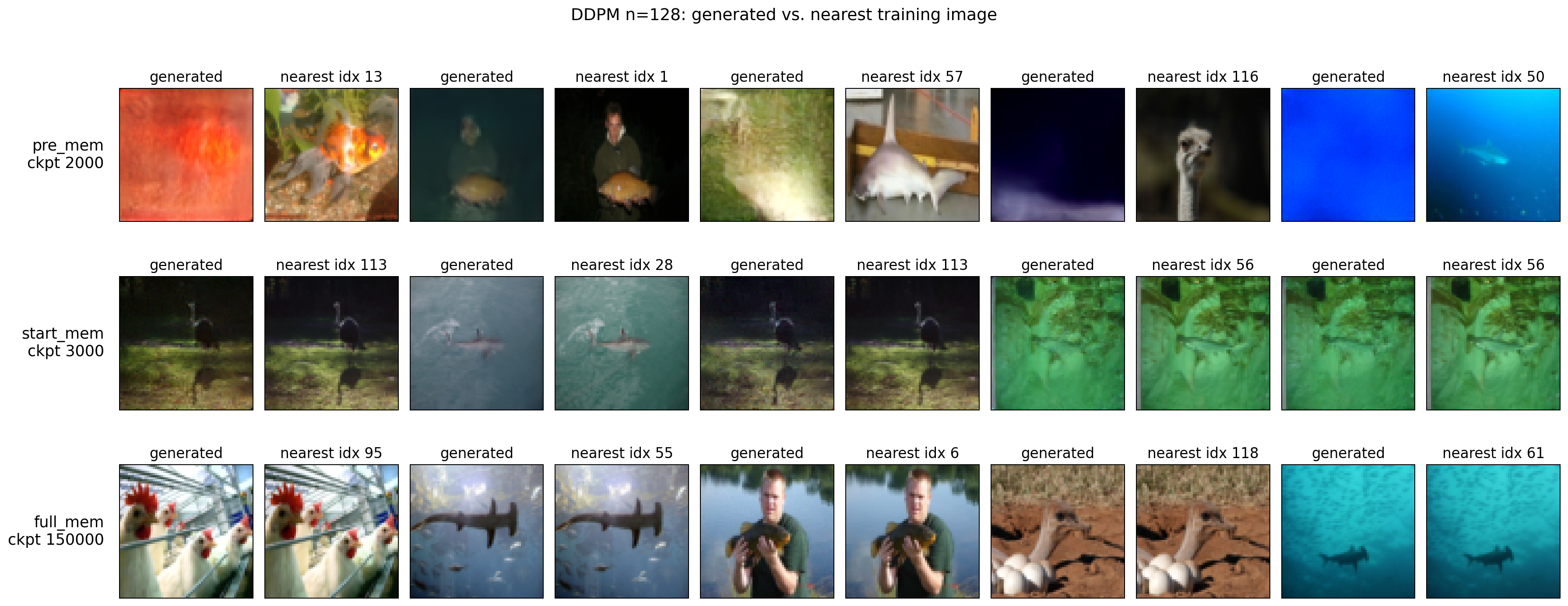}
    \caption{\textbf{Generated samples vs. nearest training image for
    ImageNet ($n=128$) DDPM.}
    pre\_mem: step $2{,}000$, $f_{\mathrm{mem}}=0$, five random samples.
    start\_mem: step $3{,}000$, $f_{\mathrm{mem}}=0.10\%$
    (10 copies of $10^4$), five randomly selected copies.
    full\_mem: step $150{,}000$, $f_{\mathrm{mem}}=99.8\%$,
    five randomly selected copies.}
    \label{fig:imagenet-stages-n128}
\end{figure}

\begin{figure}[!htbp]
    \centering
    \includegraphics[width=0.9\linewidth]{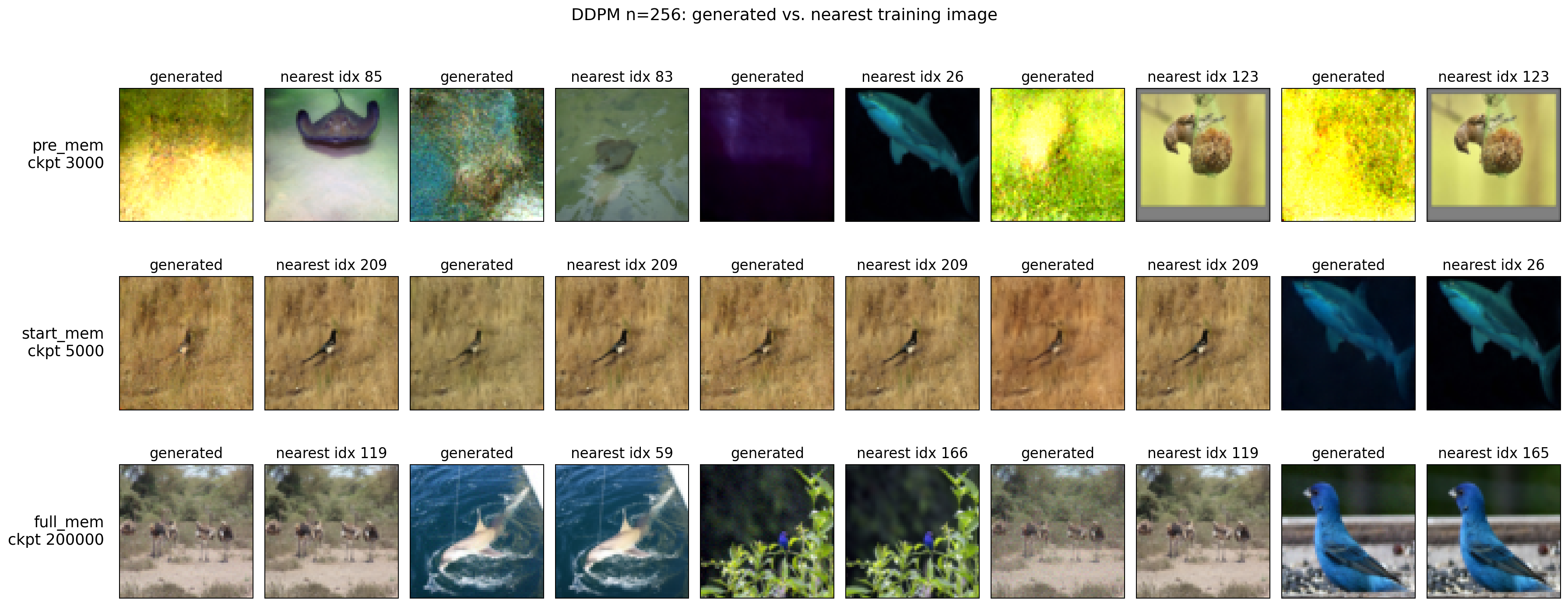}
    \caption{\textbf{Generated samples vs. nearest training image for
    ImageNet ($n=256$) DDPM.}
    pre\_mem: step $3{,}000$, $f_{\mathrm{mem}}=0$, five random samples.
    start\_mem: step $5{,}000$, $f_{\mathrm{mem}}=0.04\%$
    (4 copies of $10^4$); the four copies, followed in the rightmost
    pair by the non-copy with the smallest gap ratio.
    full\_mem: step $200{,}000$, $f_{\mathrm{mem}}=95.5\%$,
    five randomly selected copies.}
    \label{fig:imagenet-stages-n256}
\end{figure}

\begin{figure}[!htbp]
    \centering
    \includegraphics[width=0.9\linewidth]{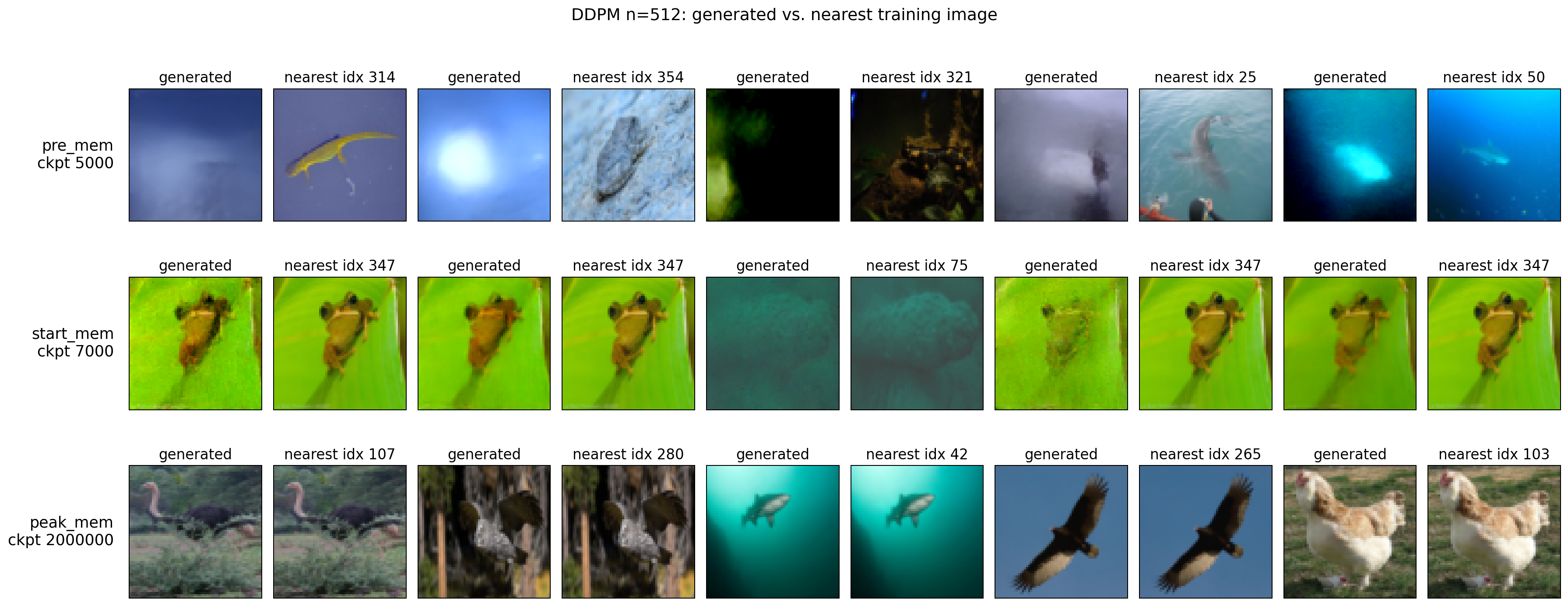}
    \caption{\textbf{Generated samples vs. nearest training image for
    ImageNet ($n=512$) DDPM.}
    pre\_mem: step $5{,}000$, $f_{\mathrm{mem}}=0$, five random samples.
    start\_mem: step $7{,}000$, $f_{\mathrm{mem}}=0.03\%$
    (3 copies of $10^4$); the three copies, followed in the two
    rightmost pairs by the non-copies with the smallest gap ratios.
    peak\_mem: step $2{,}000{,}000$, $f_{\mathrm{mem}}=85.0\%$,
    five randomly selected copies. This model never reaches the
    full-memorization threshold.}
    \label{fig:imagenet-stages-n512}
\end{figure}

\begin{figure}[!htbp]
    \centering
    \includegraphics[width=0.9\linewidth]{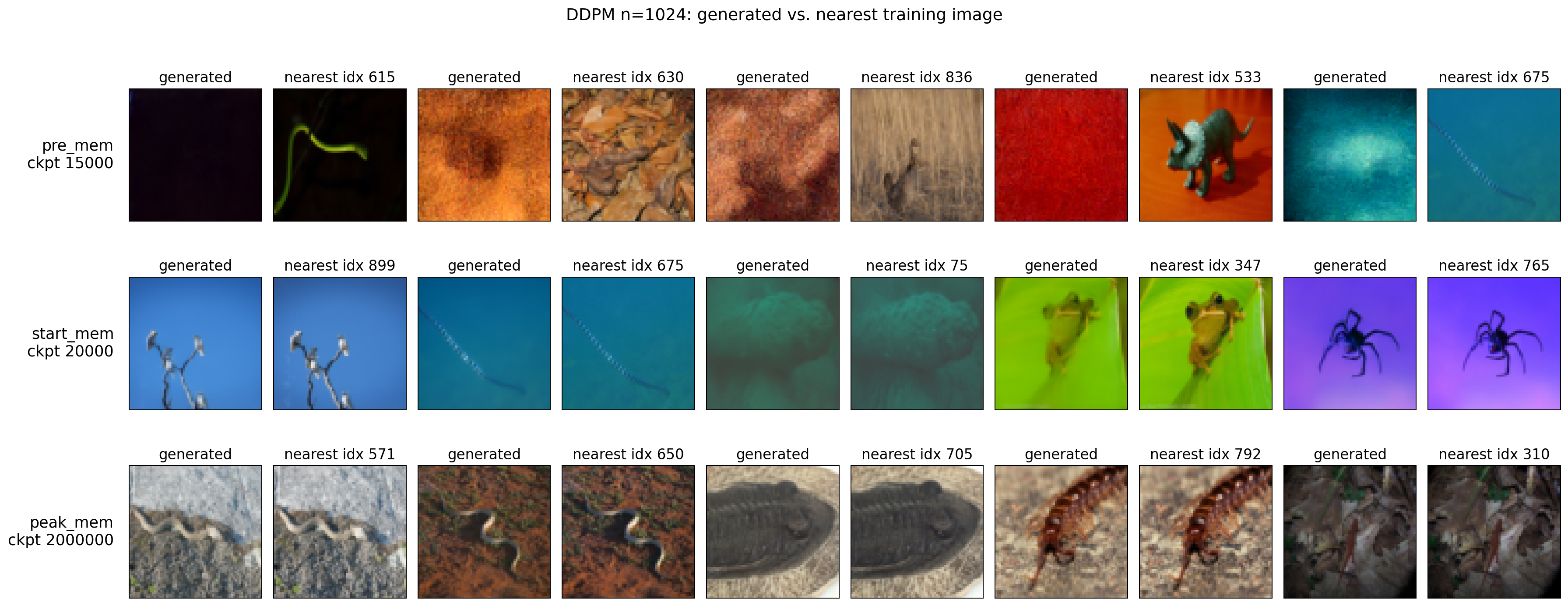}
    \caption{\textbf{Generated samples vs. nearest training image for
    ImageNet ($n=1024$) DDPM.}
    pre\_mem: step $15{,}000$, $f_{\mathrm{mem}}=0$, five random samples.
    start\_mem: step $20{,}000$, $f_{\mathrm{mem}}=0.10\%$
    (10 copies of $10^4$), five randomly selected copies.
    peak\_mem: step $2{,}000{,}000$, $f_{\mathrm{mem}}=19.9\%$,
    five randomly selected copies. This model never reaches the
    full-memorization threshold.}
    \label{fig:imagenet-stages-n1024}
\end{figure}

\begin{figure}[!htbp]
    \centering
    \includegraphics[width=0.9\linewidth]{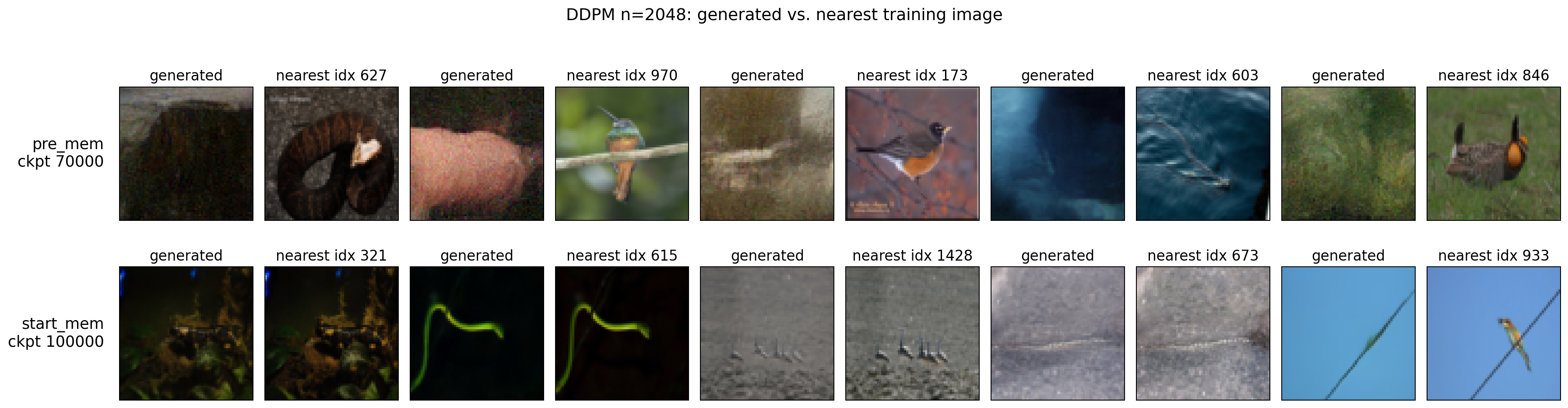}
    \caption{\textbf{Generated samples vs. nearest training image for
    ImageNet ($n=2048$) DDPM.}
    pre\_mem: step $70{,}000$, $f_{\mathrm{mem}}=0$, five random samples.
    start\_mem: step $100{,}000$, $f_{\mathrm{mem}}=0.02\%$
    (2 copies of $10^4$); the two copies, followed in the three
    rightmost pairs by the non-copies with the smallest gap ratios.
    There is no full\_mem row: $f_{\mathrm{mem}}$ peaks at $0.57\%$.}
    \label{fig:imagenet-stages-n2048}
\end{figure}

\begin{figure}[!htbp]
    \centering
    \includegraphics[width=0.9\linewidth]{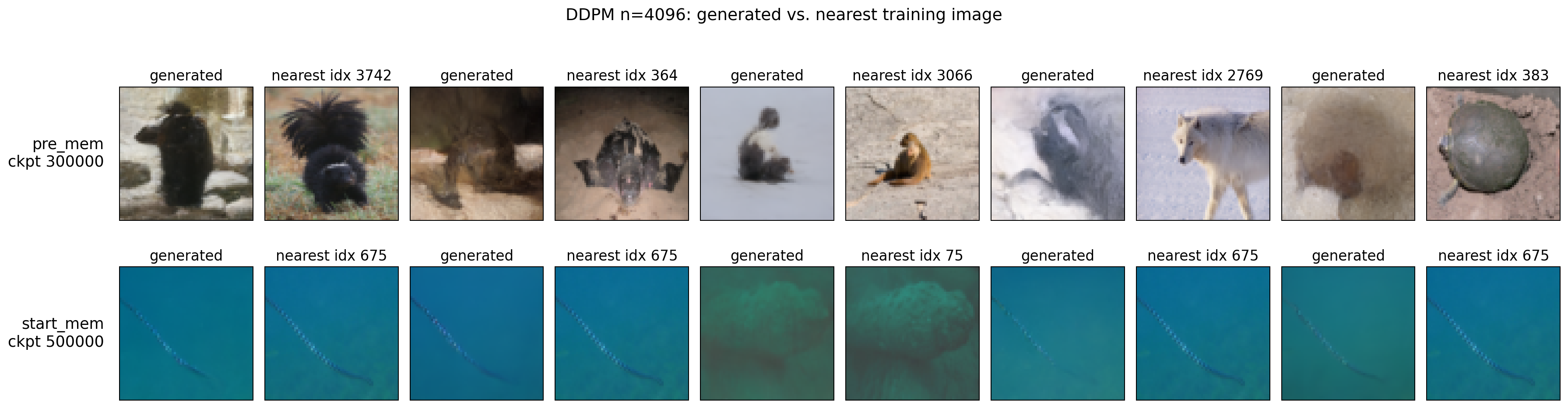}
    \caption{\textbf{Generated samples vs. nearest training image for
    ImageNet ($n=4096$) DDPM.}
    pre\_mem: step $300{,}000$, $f_{\mathrm{mem}}=0$, five random samples.
    start\_mem: step $500{,}000$, an isolated checkpoint with
    $f_{\mathrm{mem}}=0.03\%$ (3 copies of $10^4$); the three copies,
    followed in the two rightmost pairs by the non-copies with the
    smallest gap ratios. Here start\_mem labels the isolated copy event.
    No sustained memorization is observed, and there is no full\_mem row.}
    \label{fig:imagenet-stages-n4096}
\end{figure}

\begin{figure}[!htbp]
    \centering
    \includegraphics[width=0.9\linewidth]{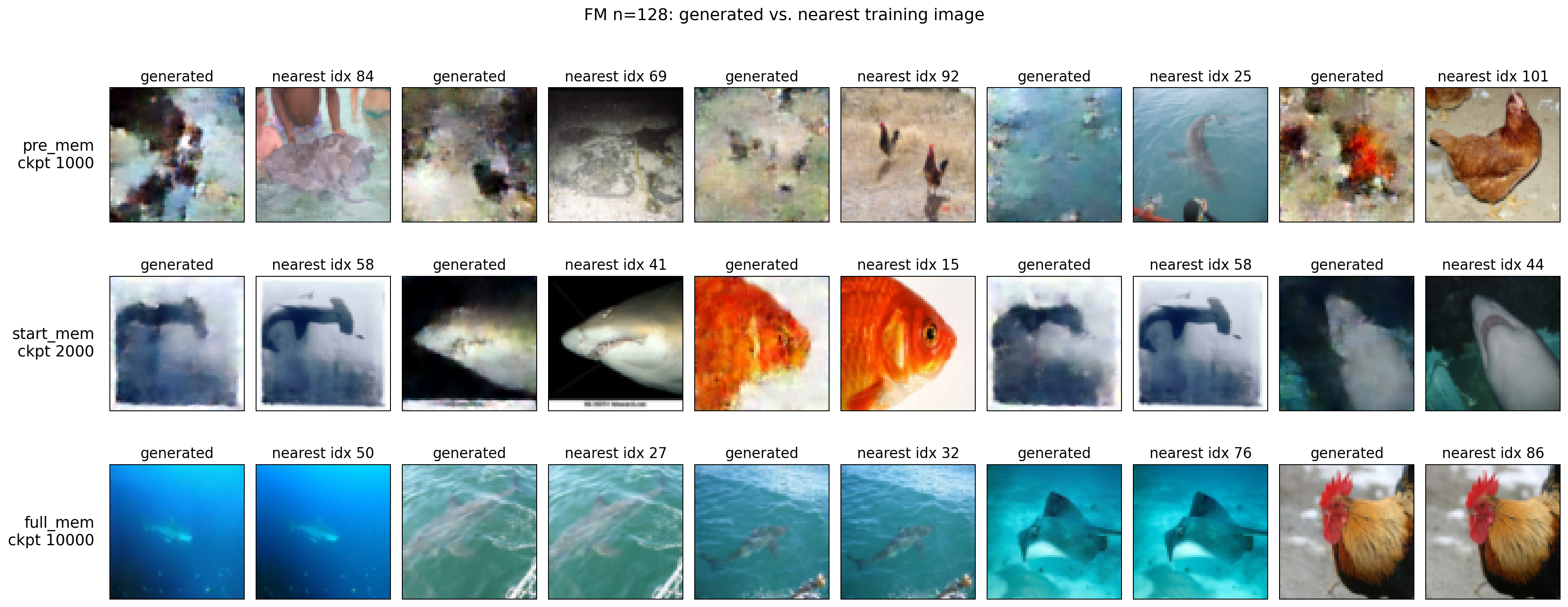}
    \caption{\textbf{Generated samples vs. nearest training image for
    ImageNet ($n=128$) flow matching.}
    pre\_mem: step $1{,}000$, $f_{\mathrm{mem}}=0$, five random samples.
    start\_mem: step $2{,}000$, $f_{\mathrm{mem}}=0.10\%$
    (10 copies of $10^4$), five randomly selected copies.
    full\_mem: step $10{,}000$, $f_{\mathrm{mem}}=95.7\%$,
    five randomly selected copies.}
    \label{fig:imagenet-stages-fm-n128}
\end{figure}

\begin{figure}[!htbp]
    \centering
    \includegraphics[width=0.9\linewidth]{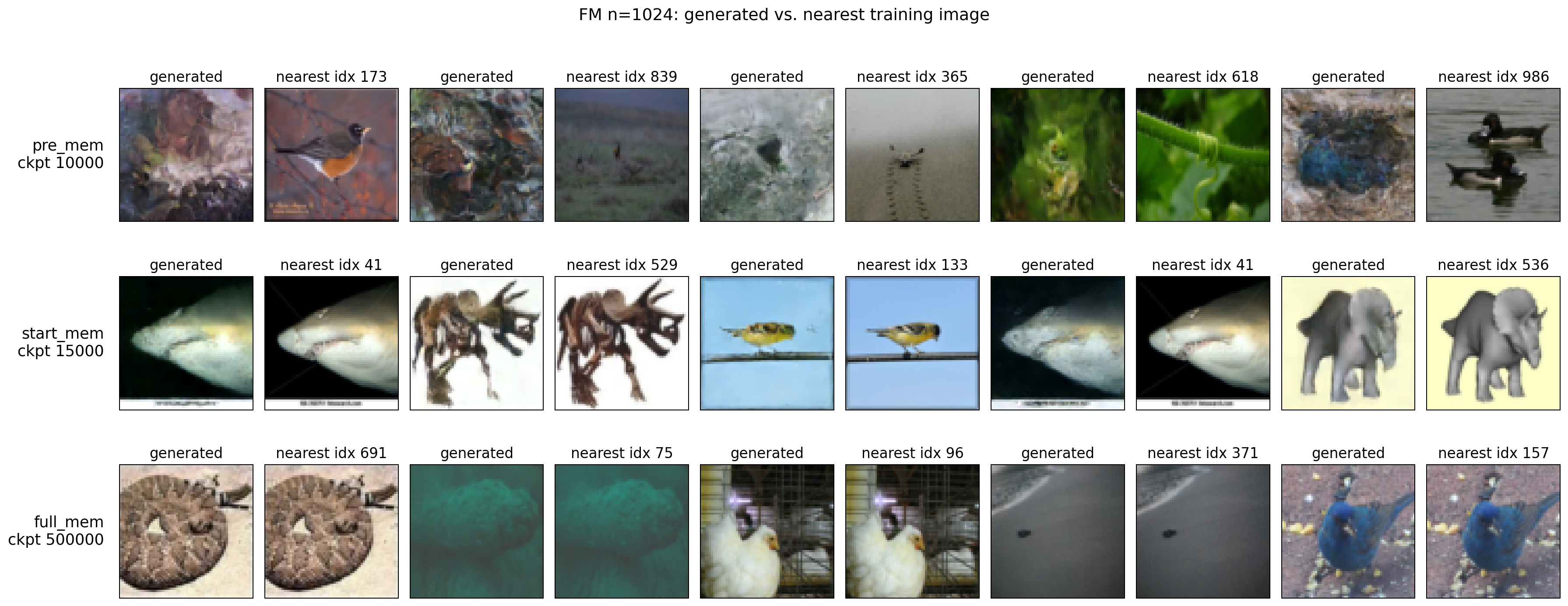}
    \caption{\textbf{Generated samples vs. nearest training image for
    ImageNet ($n=1024$) flow matching.}
    pre\_mem: step $10{,}000$, $f_{\mathrm{mem}}=0$, five random samples.
    start\_mem: step $15{,}000$, $f_{\mathrm{mem}}=0.46\%$
    (46 copies of $10^4$), five randomly selected copies.
    full\_mem: step $500{,}000$, $f_{\mathrm{mem}}=96.2\%$,
    five randomly selected copies.}
    \label{fig:imagenet-stages-fm-n1024}
\end{figure}

\clearpage

\section{Additional Experiments}
\label{app: additional}
We provide additional experiments on Imagenet and CelebA. Figures \ref{fig:celeba_128_mem_curv}-\ref{fig:imagenet_10240_mem_curv} plot the average training input loss curvature over diffusion time and fraction of generated samples that are considered memorized with memorization threshold $k=1/3$ vs. training step for DDPM noise prediction and flow matching for the datasets CelebA with training set size $n\in\{128,1024\}$ and Imagenet with training set size $n\in\{128,1024,10240\}$, respectively. DDPM and flow matching on CelebA with $n=128$ is a memorizing model, since fraction memorization starts increasing rapidly to $100\%$ around training step $\tau =10^3$, whereas DDPM and flow matching on CelebA with $n=10240$ is a generalizing model with fraction memorization increasing much later around training step $\tau=10^5$ and never reaching full memorization of $100\%$ before training ends. We observe that the average input loss curvature increases more drastically for the memorizing model as compared to the generalizing model. A similar trend can be seen for the DDPM and flow matching models for Imagenet as training size $n$ increases. 

For the flow matching model on CelebA with training size $n=128$, Figure \ref{fig:celeba_fm_128_train_val_curv} shows the average training and validation input loss curvatures over diffusion time vs. training step, as well as the relative training and validation curvature gaps, which is calculated as $|\mathbb{E}_{x \sim \text{train}}[\|\nabla^{2}_{x}\ell(\phi(\tau), x)\|^{2}_{F}] - \mathbb{E}_{x \sim \text{val.}}[\|\nabla^{2}_{x}\ell(\phi(\tau), x)\|^{2}_{F}]|/\mathbb{E}_{x \sim \text{train}}[\|\nabla^{2}_{x}\ell(\phi(\tau), x)\|^{2}_{F}]$ and $|\mathbb{E}_{x \sim \text{train}}[\|\nabla^{2}_{x}\ell(\phi(\tau), x)\|^{2}_{F}] - \mathbb{E}_{x \sim \text{val.}}[\|\nabla^{2}_{x}\ell(\phi(\tau), x)\|^{2}_{F}]|/\mathbb{E}_{x \sim \text{val}}[\|\nabla^{2}_{x}\ell(\phi(\tau), x)\|^{2}_{F}]$, vs. training step. We observe that the training curvature separates from the validation curvature right before training step $\tau=10^3$, which is around where memorization begins, as shown in Figure \ref{fig:celeba_128_mem_curv}.

Figure \ref{fig:imagenet_10240_mem} shows the fraction memorization score for DDPM noise prediction and flow matching on Imagenet $(n=10240)$ with different values for the memorization threshold $k \in\{1/3,0.8,0.9\}$. There is no trend in the nearest-neighbor gap ratios for DDPM even when $k$ is increased. However, for flow matching, when $k$ is large, the nearest-neighbor gap ratios increase as training continues. Nonetheless, this pattern in the gap ratios does not indicate that this model is a memorizing model, since for $k=1/3$, the fraction memorized is approximately $0\%$ for all training steps. 

Figure \ref{fig:ddpm_sample_imagenet_10240} shows the average training curvature over diffusion time and fraction memorized with $k=1/3$ vs. training step for DDPM sample prediction on Imagenet $(n=10240)$ as well as the fraction memorization with $k \in \{1/3,0.8,0.9\}$. This model is a generalizing model, since for $k=1/3$, approximately $0\%$ of the generated samples are considered memorized for all training steps, even though for large values of $k$, the nearest-neighbor gap ratios increase as training steps increase. We also observe that the average input loss curvature barely changes across training steps, validating the connection between memorization and curvature. 

In Figure \ref{fig:celeba_vs_imagenet_curv}, the average training input loss curvature over diffusion time vs. training step is plotted for DDPM noise prediction and flow matching on the CelebA and Imagenet datasets with training set size $n=10240$. CelebA has larger curvature values than Imagenet for both models, which correspond to higher memorization scores as seen in Figures \ref{fig:celeba_10240_mem_curv} and \ref{fig:imagenet_10240_mem_curv}. In particular, DDPM and flow matching for CelebA $(n=10240)$ have fraction memorization scores of around $23\%$ and $40\%$, respectively, at the end of training, whereas DDPM and flow matching for Imagenet $(n=10240)$ have fraction memorization scores of around $0\%$ and $0.1\%$, respectively, at the end of training. 

Figure \ref{fig:ddpm_1024_leave10out_mem_curv} compares two DDPM noise prediction models on Imagenet, one that is trained on a training set of size $n=1024$ and another that is trained on the same training set with the top $10$ memorized samples from the full $n=1024$ model removed. The figure plots fraction memorization and average training curvature over diffusion time vs. training step for both the full and leave-10-out models. The leave-10-out model has a lower memorization score at the end of training and lower input loss curvature for all training steps, indicating that it is worthwhile to remove the top memorized samples from a model to improve generalization. Figure \ref{fig:fm_1024_leave10out_mem_curv} shows similar results for flow matching. We repeat these experiments for Imagenet $n=128$ in Figures \ref{fig:ddpm_128_leave10out_mem_curv} and \ref{fig:fm_128_leave10out_mem_curv}. Figures \ref{fig:fm_128_loss_diff_curv} and \ref{fig:fm_1024_loss_diff_curv} plot the difference between the per sample losses of the full flow matching models and the leave-10-out models vs. training step, as well as the per sample curvatures of the leave-10-out models, for each of the top $10$ left-out samples.

\begin{figure}[h]
    \centering
    \begin{subfigure}[b]{0.49\linewidth}
        \centering
        \includegraphics[width=\linewidth]{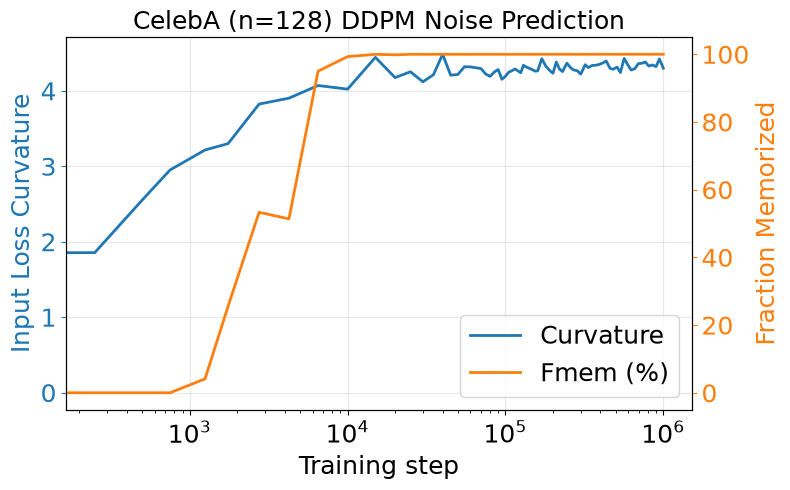}
        \caption{}
        \label{fig:left}
    \end{subfigure}
    \hfill
    \begin{subfigure}[b]{0.49\linewidth}
        \centering
        \includegraphics[width=\linewidth]{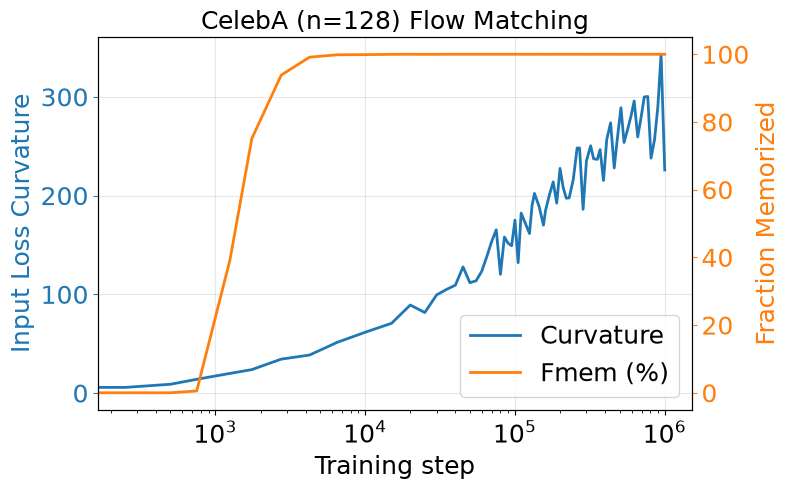}
        \caption{}
        \label{fig:right}
    \end{subfigure}
    \caption{\textbf{Curvature and memorization for CelebA (n=128)}: average input loss curvature over diffusion time and training set and fraction memorized (\%) with memorization threshold $k=1/3$ vs. training step for (a) DDPM noise prediction and (b) flow matching.}
    \label{fig:celeba_128_mem_curv}
\end{figure}

\begin{figure}[h]
    \centering
    \begin{subfigure}[b]{0.49\linewidth}
        \centering
        \includegraphics[width=\linewidth]{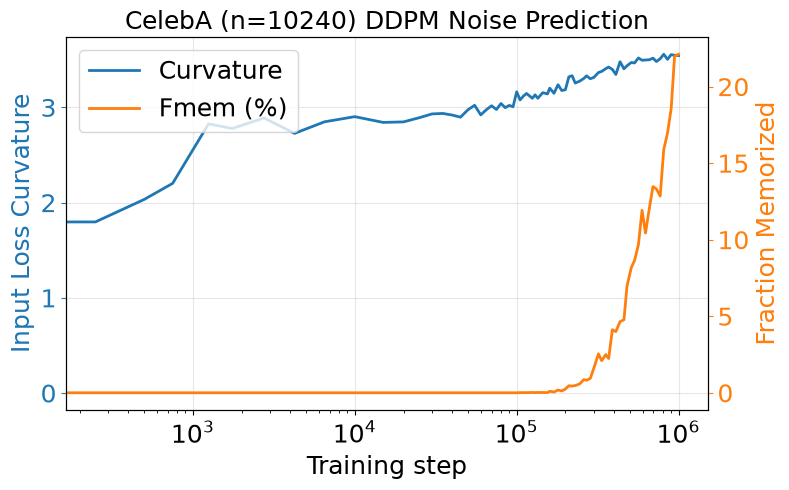}
        \caption{}
        \label{fig:left}
    \end{subfigure}
    \hfill
    \begin{subfigure}[b]{0.49\linewidth}
        \centering
        \includegraphics[width=\linewidth]{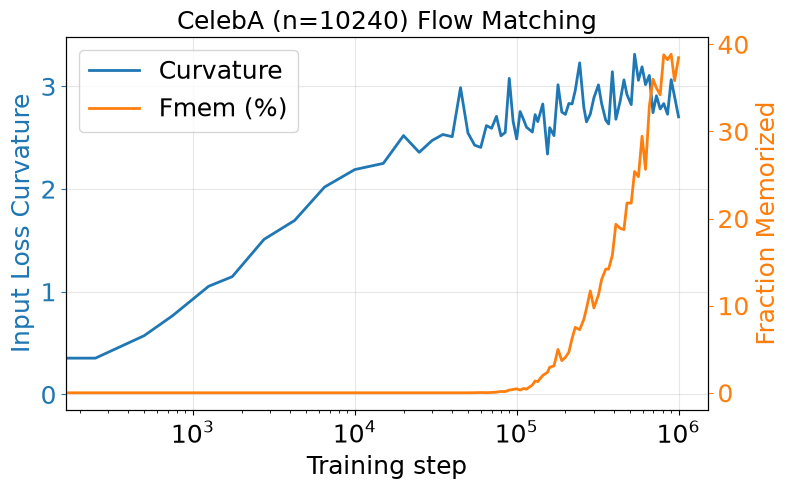}
        \caption{}
        \label{fig:right}
    \end{subfigure}
    \caption{\textbf{Curvature and memorization for CelebA (n=10240)}: average input loss curvature over diffusion time and training set and fraction memorized (\%) with memorization threshold $k=1/3$ vs. training step for (a) DDPM noise prediction and (b) flow matching.}
    \label{fig:celeba_10240_mem_curv}
\end{figure}

\begin{figure}[h]
    \centering
    \begin{subfigure}[b]{0.49\linewidth}
        \centering
        \includegraphics[width=\linewidth]{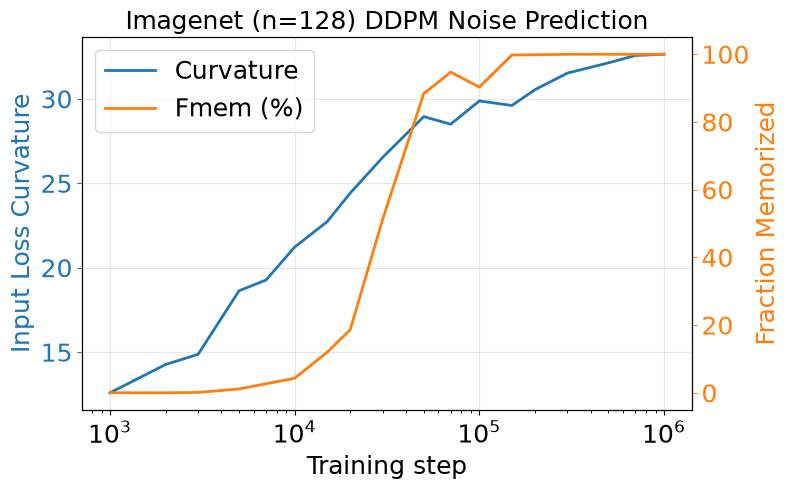}
        \caption{}
        \label{fig:left}
    \end{subfigure}
    \hfill
    \begin{subfigure}[b]{0.49\linewidth}
        \centering
        \includegraphics[width=\linewidth]{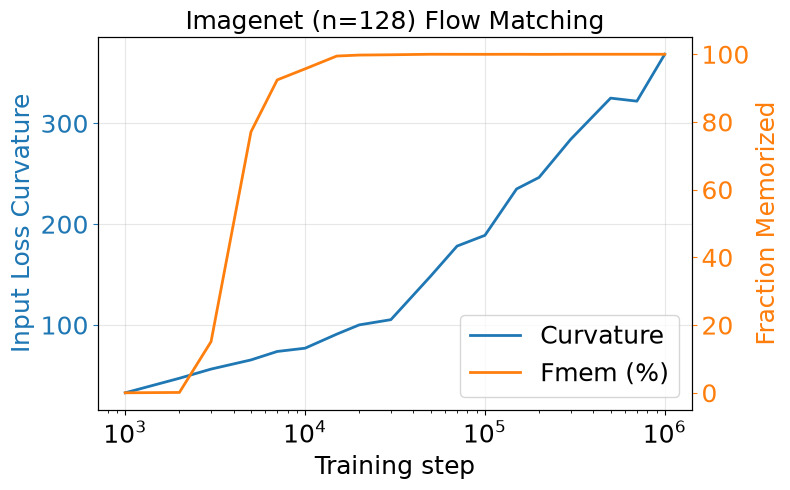}
        \caption{}
        \label{fig:right}
    \end{subfigure}
    \caption{\textbf{Curvature and memorization for Imagenet (n=128)}: average input loss curvature over diffusion time and training set and fraction memorized (\%) with memorization threshold $k=1/3$ vs. training step for (a) DDPM noise prediction and (b) flow matching.}
    \label{fig:imagenet_128_mem_curv}
\end{figure}

\begin{figure}[h]
    \centering
    \begin{subfigure}[b]{0.49\linewidth}
        \centering
        \includegraphics[width=\linewidth]{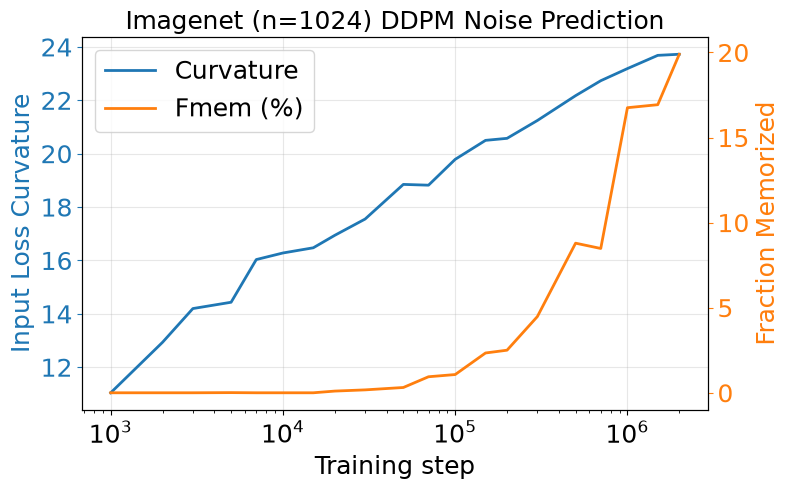}
        \caption{}
        \label{fig:left}
    \end{subfigure}
    \hfill
    \begin{subfigure}[b]{0.49\linewidth}
        \centering
        \includegraphics[width=\linewidth]{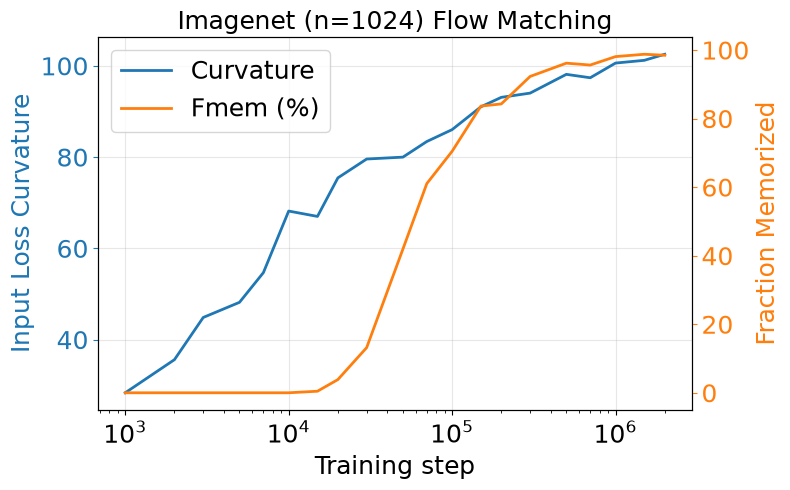}
        \caption{}
        \label{fig:right}
    \end{subfigure}
    \caption{\textbf{Curvature and memorization for Imagenet (n=1024)}: average input loss curvature over diffusion time and training set and fraction memorized (\%) with memorization threshold $k=1/3$ vs. training step for (a) DDPM noise prediction and (b) flow matching.}
    \label{fig:imagenet_1024_mem_curv}
\end{figure}

\begin{figure}[h]
    \centering
    \begin{subfigure}[b]{0.49\linewidth}
        \centering
        \includegraphics[width=\linewidth]{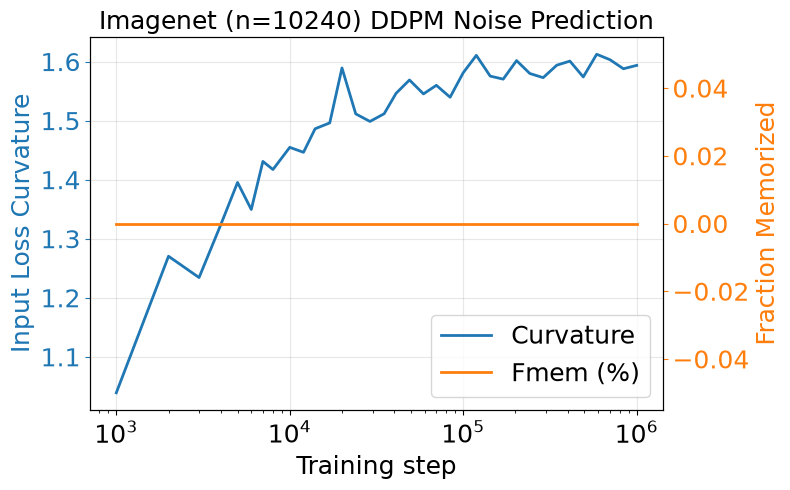}
        \caption{}
        \label{fig:left}
    \end{subfigure}
    \hfill
    \begin{subfigure}[b]{0.49\linewidth}
        \centering
        \includegraphics[width=\linewidth]{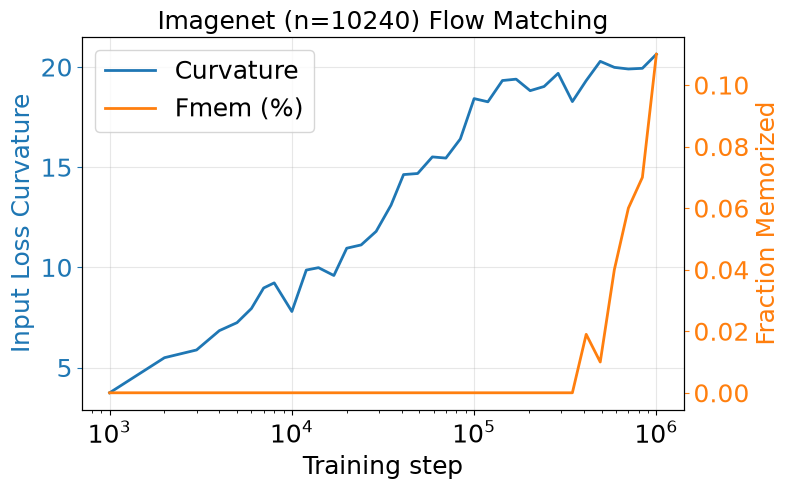}
        \caption{}
        \label{fig:right}
    \end{subfigure}
    \caption{\textbf{Curvature and memorization for Imagenet (n=10240)}: average input loss curvature over diffusion time and training set and fraction memorized (\%) with memorization threshold $k=1/3$ vs. training step for (a) DDPM noise prediction and (b) flow matching.}
    \label{fig:imagenet_10240_mem_curv}
\end{figure}

\begin{figure}[h]
    \centering
    \includegraphics[width=0.96\linewidth]{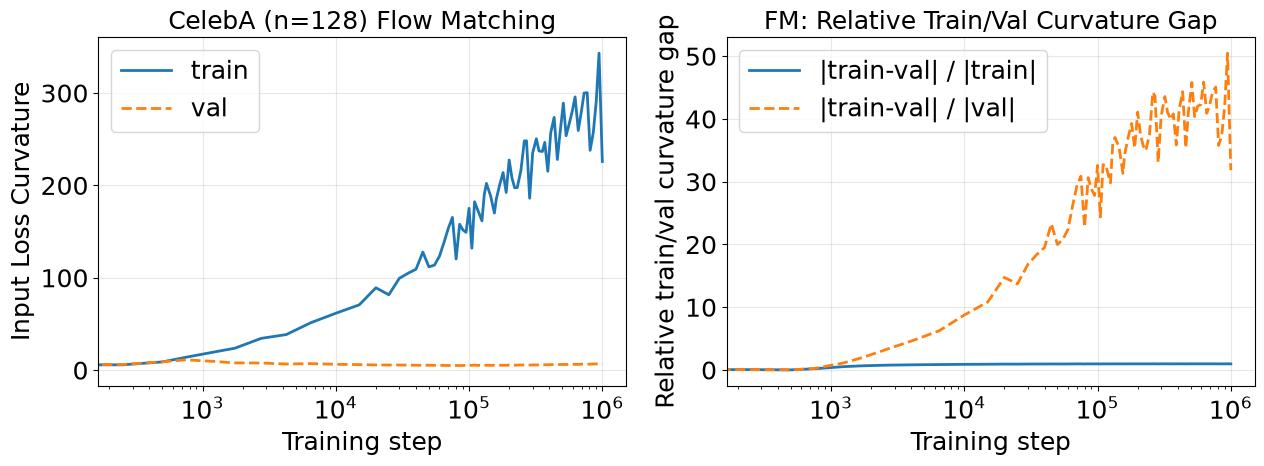}
    \caption{\textbf{Training vs. validation curvature for CelebA (n=128) flow matching}: (left) average input loss curvature over diffusion time and training (blue) / validation (orange) set, (right) relative training/validation curvature gap calculated as $|\mathbb{E}_{x \sim \text{train}}[\|\nabla^{2}_{x}\ell(\phi(\tau), x)\|^{2}_{F}] - \mathbb{E}_{x \sim \text{val.}}[\|\nabla^{2}_{x}\ell(\phi(\tau), x)\|^{2}_{F}]|/\mathbb{E}_{x \sim \text{train}}[\|\nabla^{2}_{x}\ell(\phi(\tau), x)\|^{2}_{F}]$ (blue) and $|\mathbb{E}_{x \sim \text{train}}[\|\nabla^{2}_{x}\ell(\phi(\tau), x)\|^{2}_{F}] - \mathbb{E}_{x \sim \text{val.}}[\|\nabla^{2}_{x}\ell(\phi(\tau), x)\|^{2}_{F}]|/\mathbb{E}_{x \sim \text{val}}[\|\nabla^{2}_{x}\ell(\phi(\tau), x)\|^{2}_{F}]$ (orange).}
    \label{fig:celeba_fm_128_train_val_curv}
\end{figure}

\begin{figure}[h]
    \centering
    \begin{subfigure}[b]{0.49\linewidth}
        \centering
        \includegraphics[width=\linewidth]{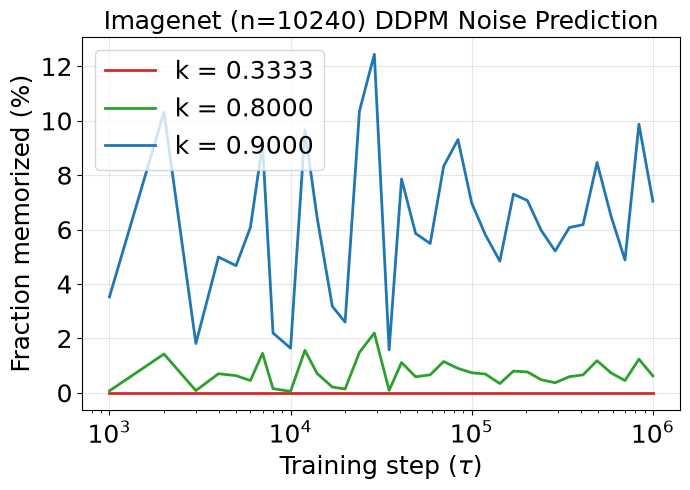}
        \caption{}
        \label{fig:left}
    \end{subfigure}
    \hfill
    \begin{subfigure}[b]{0.49\linewidth}
        \centering
        \includegraphics[width=\linewidth]{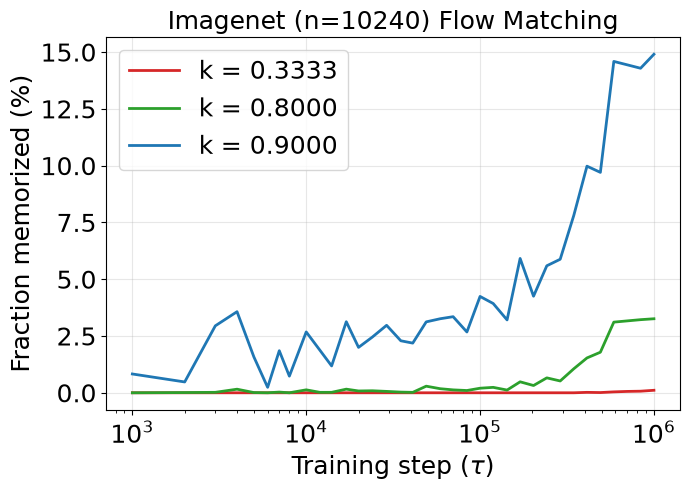}
        \caption{}
        \label{fig:right}
    \end{subfigure}
    \caption{\textbf{Memorization for Imagenet (n=10240)}: fraction memorized (\%) with memorization threshold $k \in \{1/3, 0.8, 0.9\}$ vs. training step for (a) DDPM noise prediction and (b) flow matching.}
    \label{fig:imagenet_10240_mem}
\end{figure}

\begin{figure}[h]
    \centering
    \begin{subfigure}[b]{0.52\linewidth}
        \centering
        \includegraphics[width=\linewidth]{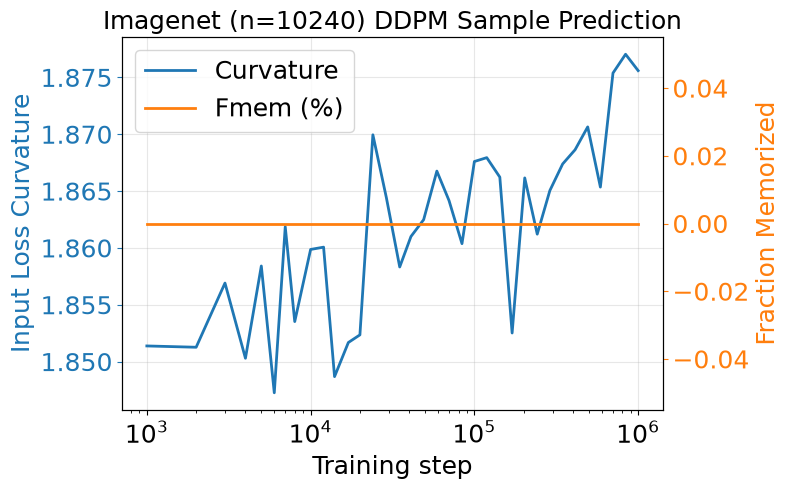}
        \caption{}
        \label{fig:left}
    \end{subfigure}
    \hfill
    \begin{subfigure}[b]{0.45\linewidth}
        \centering
        \includegraphics[width=\linewidth]{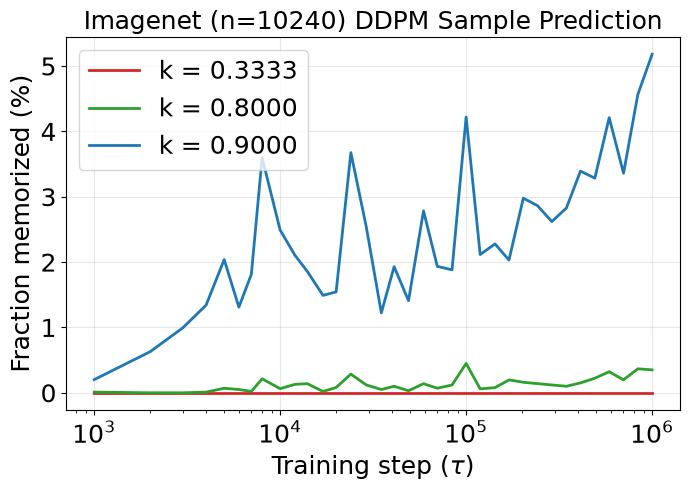}
        \caption{}
        \label{fig:right}
    \end{subfigure}
    \caption{\textbf{Curvature and memorization for Imagenet (n=10240) DDPM Sample Prediction}: (a) average input loss curvature over diffusion time and training set and fraction memorized (\%)
    with memorization threshold $k=1/3$ vs. training step, and (b) fraction memorized (\%) with memorization threshold $k\in\{1/3,0.8,0.9\}$ vs. training step.} \label{fig:ddpm_sample_imagenet_10240}
\end{figure}

\begin{figure}[h]
    \centering
    \begin{subfigure}[b]{0.49\linewidth}
        \centering
        \includegraphics[width=\linewidth]{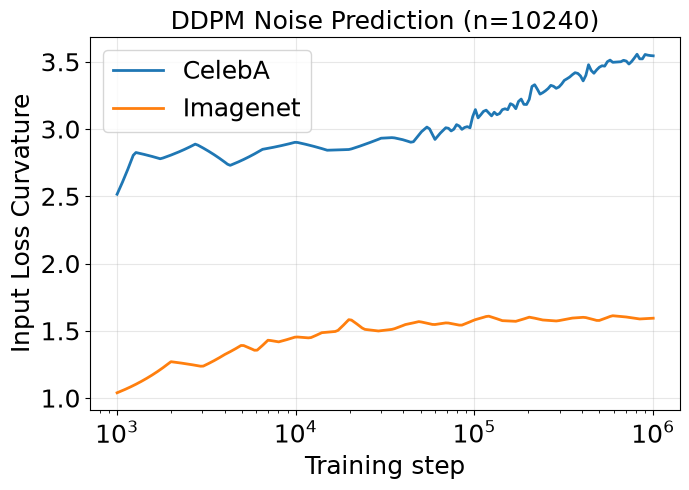}
        \caption{}
        \label{fig:left}
    \end{subfigure}
    \hfill
    \begin{subfigure}[b]{0.49\linewidth}
        \centering
        \includegraphics[width=\linewidth]{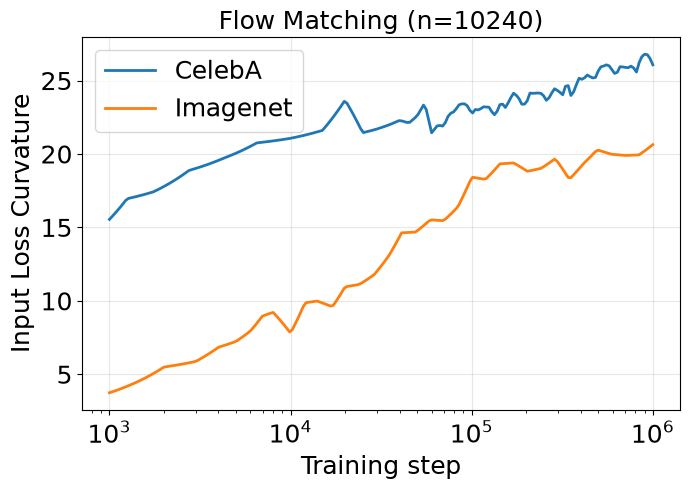}
        \caption{}
        \label{fig:right}
    \end{subfigure}
    \caption{\textbf{CelebA vs. Imagenet (n=10240) Curvature}: average input loss curvature over diffusion time and training set vs. training step for (a) DDPM noise prediction and (b) flow matching.}
    \label{fig:celeba_vs_imagenet_curv}
\end{figure}

\begin{figure}[h]
    \centering
    \begin{subfigure}[b]{0.49\linewidth}
        \centering
        \includegraphics[width=\linewidth]{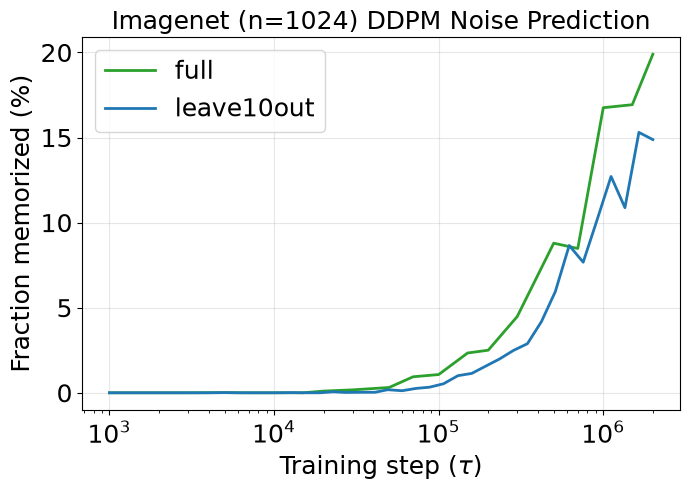}
        \caption{}
        \label{fig:left}
    \end{subfigure}
    \hfill
    \begin{subfigure}[b]{0.49\linewidth}
        \centering
        \includegraphics[width=\linewidth]{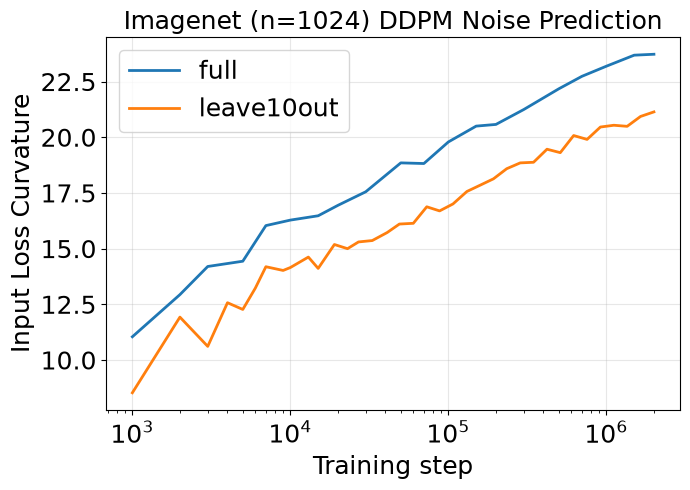}
        \caption{}
        \label{fig:right}
    \end{subfigure}
    \caption{\textbf{Memorization and curvature for Imagenet (n=1024) DDPM noise prediction - full vs. leave-10-out training}: (a) fraction memorized (\%) with memorization threshold $k=1/3$ vs. training step and (b) average input loss curvature over diffusion time and training set vs. training step for full $n=1024$ training set and leave-10-out training set with the top $10$ memorized training samples removed.}
    \label{fig:ddpm_1024_leave10out_mem_curv}
\end{figure}

\begin{figure}[h]
    \centering
    \begin{subfigure}[b]{0.49\linewidth}
        \centering
        \includegraphics[width=\linewidth]{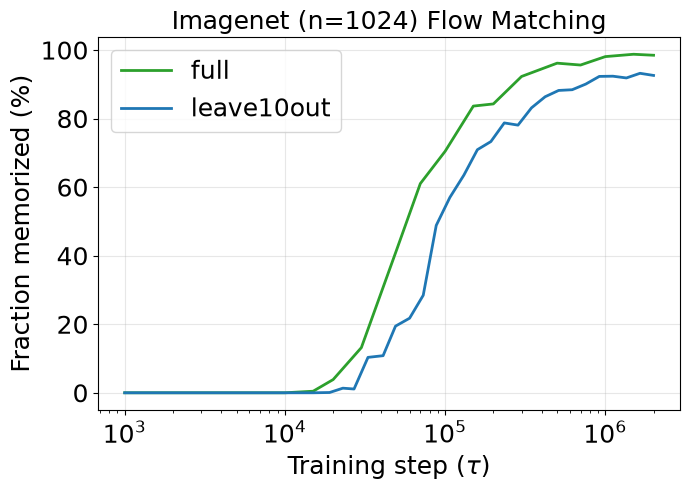}
        \caption{}
        \label{fig:left}
    \end{subfigure}
    \hfill
    \begin{subfigure}[b]{0.49\linewidth}
        \centering
        \includegraphics[width=\linewidth]{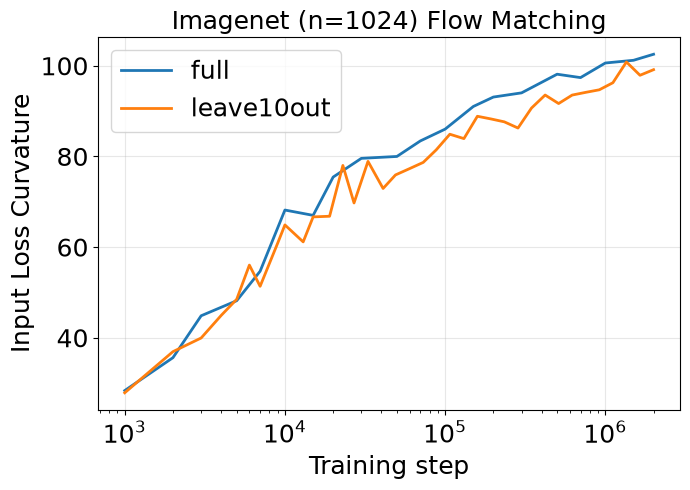}
        \caption{}
        \label{fig:right}
    \end{subfigure}
    \caption{\textbf{Memorization and curvature for Imagenet (n=1024) flow matching - full vs. leave-10-out training}: (a) fraction memorized (\%) with memorization threshold $k=1/3$ vs. training step and (b) average input loss curvature over diffusion time and training set vs. training step for full $n=1024$ training set and leave-10-out training set with the top $10$ memorized training samples removed.}
    \label{fig:fm_1024_leave10out_mem_curv}
\end{figure}

\begin{figure}[h]
    \centering
    \begin{subfigure}[b]{0.49\linewidth}
        \centering
        \includegraphics[width=\linewidth]{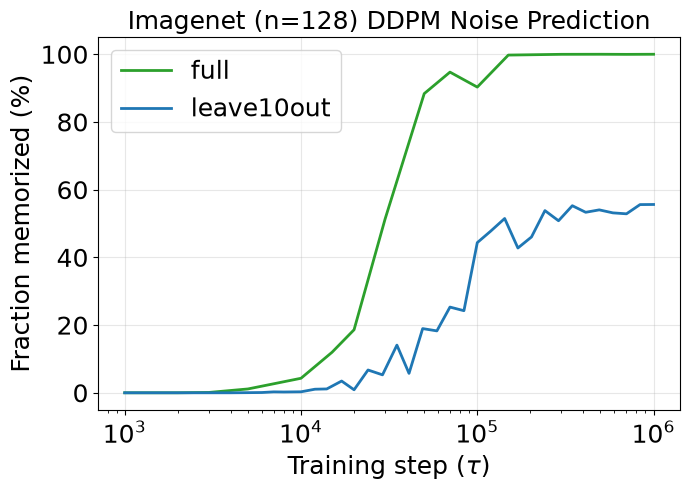}
        \caption{}
        \label{fig:left}
    \end{subfigure}
    \hfill
    \begin{subfigure}[b]{0.49\linewidth}
        \centering
        \includegraphics[width=\linewidth]{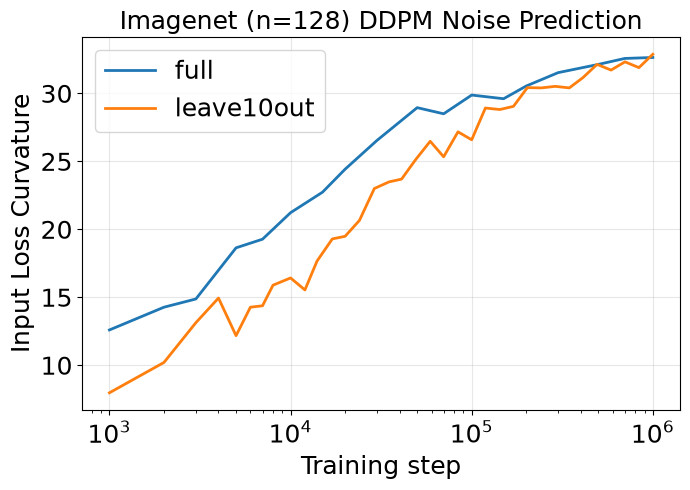}
        \caption{}
        \label{fig:right}
    \end{subfigure}
    \caption{\textbf{Memorization and curvature for Imagenet (n=128) DDPM noise prediction - full vs. leave-10-out training}: (a) fraction memorized (\%) with memorization threshold $k=1/3$ vs. training step and (b) average input loss curvature over diffusion time and training set vs. training step for full $n=128$ training set and leave-10-out training set with the top $10$ memorized training samples removed.}
    \label{fig:ddpm_128_leave10out_mem_curv}
\end{figure}

\begin{figure}[h]
    \centering
    \begin{subfigure}[b]{0.49\linewidth}
        \centering
        \includegraphics[width=\linewidth]{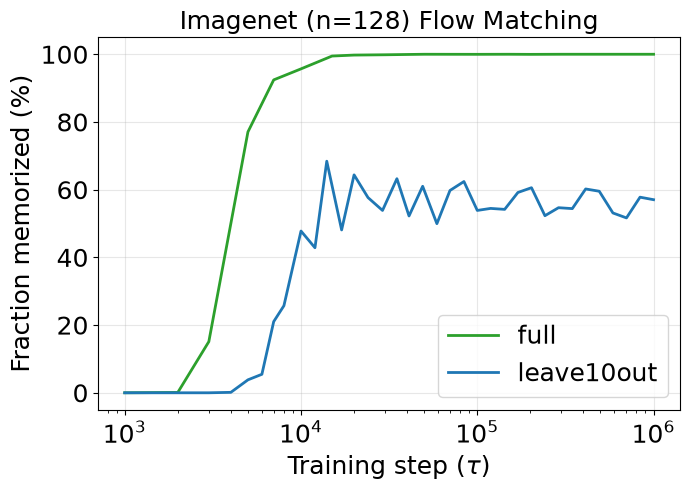}
        \caption{}
        \label{fig:left}
    \end{subfigure}
    \hfill
    \begin{subfigure}[b]{0.49\linewidth}
        \centering
        \includegraphics[width=\linewidth]{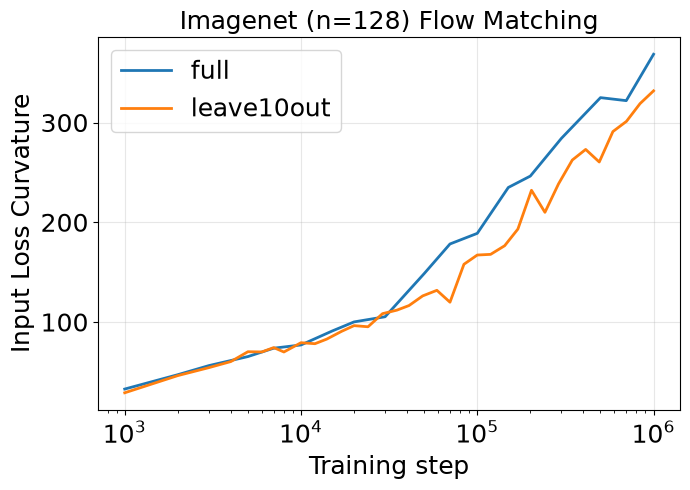}
        \caption{}
        \label{fig:right}
    \end{subfigure}
    \caption{\textbf{Memorization and curvature for Imagenet (n=128) flow matching - full vs. leave-10-out training}: (a) fraction memorized (\%) with memorization threshold $k=1/3$ vs. training step and (b) average input loss curvature over diffusion time and training set vs. training step for full $n=128$ training set and leave-10-out training set with the top $10$ memorized training samples removed.}
    \label{fig:fm_128_leave10out_mem_curv}
\end{figure}

\begin{figure}[h]
    \centering
    \includegraphics[width=0.96\linewidth]{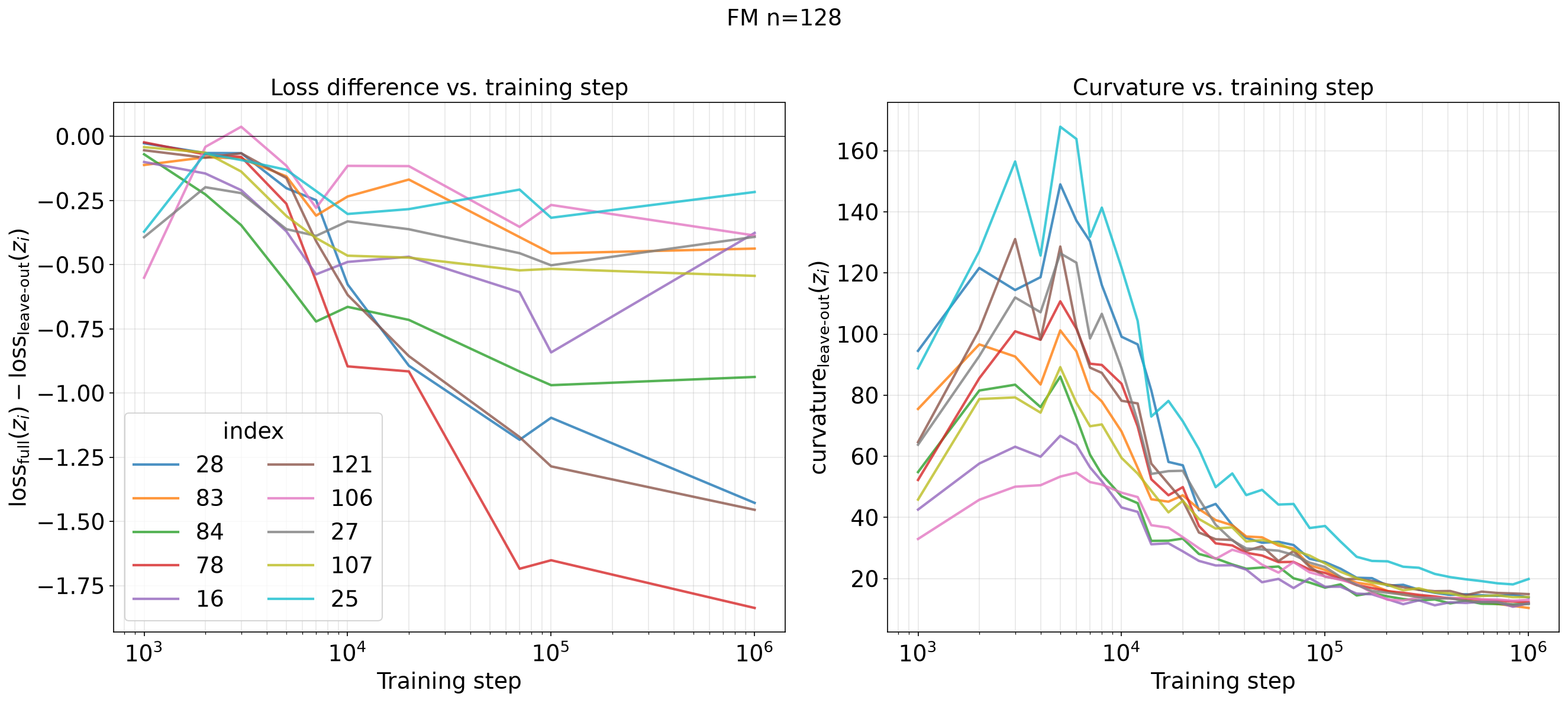}
    \caption{\textbf{Loss difference and curvature for Imagenet (n=128) flow matching - leave-10-out-training}: (left) difference between per sample loss of full model and leave-10-out model vs. training step for each of the 10 left-out training samples, (right) per sample curvature of the leave-10-out model vs. training step for each of the 10 left-out training samples.}
    \label{fig:fm_128_loss_diff_curv}
\end{figure}

\begin{figure}[h]
    \centering
    \includegraphics[width=0.96\linewidth]{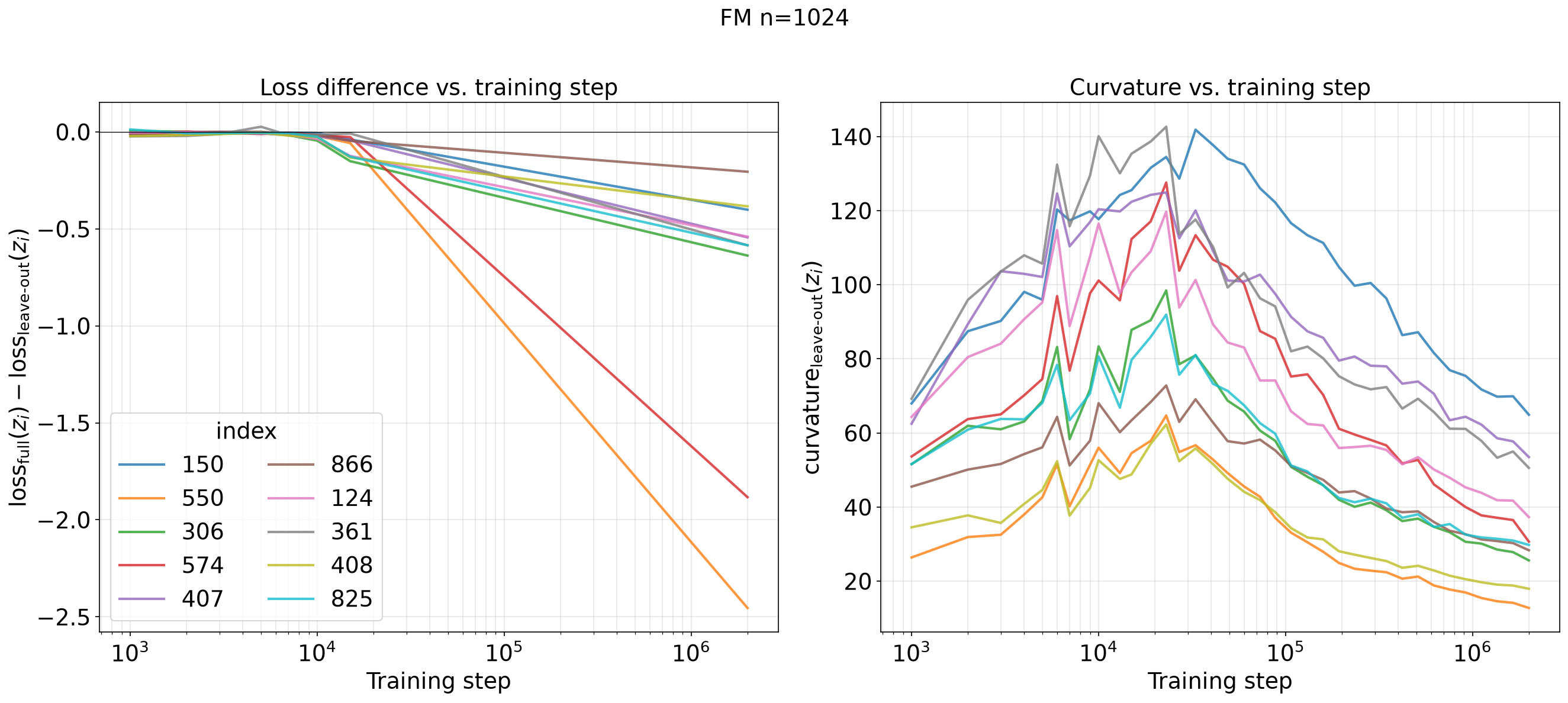}
    \caption{\textbf{Loss difference and curvature for Imagenet (n=1024) flow matching - leave-10-out-training}: (left) difference between per sample loss of full model and leave-10-out model vs. training step for each of the 10 left-out training samples, (right) per sample curvature of the leave-10-out model vs. training step for each of the 10 left-out training samples.}
    \label{fig:fm_1024_loss_diff_curv}
\end{figure}

\begin{figure}[h]
    \centering
    \begin{subfigure}[b]{0.49\linewidth}
        \centering
        \includegraphics[width=\linewidth]{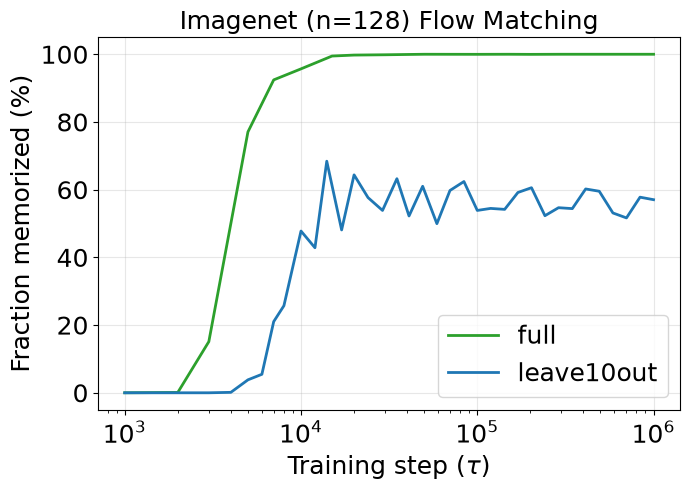}
        \caption{}
        \label{fig:left}
    \end{subfigure}
    \hfill
    \begin{subfigure}[b]{0.49\linewidth}
        \centering
        \includegraphics[width=\linewidth]{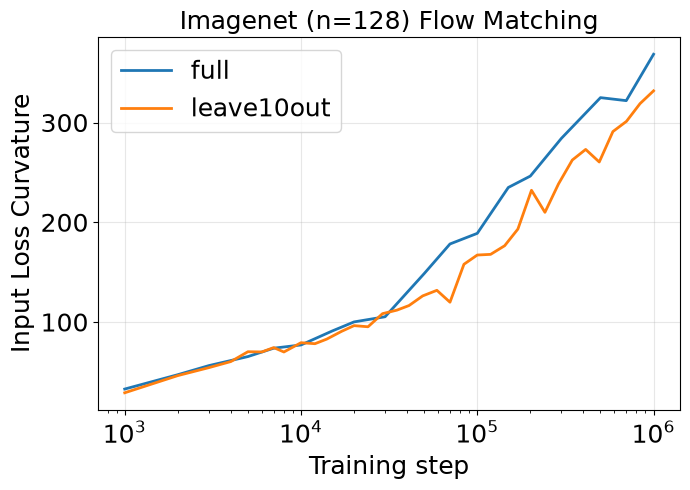}
        \caption{}
        \label{fig:right}
    \end{subfigure}
    \caption{\textbf{Memorization and curvature for Imagenet (n=128) flow matching - full vs. leave-10-out training}: (a) fraction memorized (\%) with memorization threshold $k=1/3$ vs. training step and (b) average input loss curvature over diffusion time and training set vs. training step for full $n=128$ training set and leave-10-out training set with the top $10$ memorized training samples removed.}
    \label{fig: fm_128_leave10out_mem_curv}
\end{figure}

\section{Additional ImageNet DDPM experiments}
\label{app:imagenet-ddpm}

\subsection{Experimental Details}

We use $64\times64$ RGB ImageNet images and nested training sets of size
$n\in\{128,256,512,1024,2048,4096\}$, taking the first $n$ images in a fixed
dataset ordering. We evaluate all models on the same disjoint set of
$1{,}024$ held-out images. Each training size corresponds to one DDPM trained
with the noise-prediction objective. The U-Net has base width $64$ and channel
multipliers $(1,2,4,4)$. We use Adam with learning rate $10^{-4}$ and a linear
noise schedule with $1{,}000$ diffusion steps. We train all models for $10^6$
updates and continue training those with $n\in\{512,1024,2048,4096\}$ for
another $10^6$ updates to examine their behavior over longer training.

For the curvature, gradient, and loss measurements, we convert the predicted
noise to a clean-image estimate $\hat{x}_0$ and evaluate the mean squared
reconstruction error
\[
    \ell_\tau(x;t,\epsilon)
    = \frac{1}{d}\left\|\hat{x}_0(x_t,t;\phi(\tau))-x\right\|_2^2,
    \qquad d=3\cdot64^2,
\]
where $d$ is the number of scalar image values across the three RGB channels,
so $1/d$ averages the squared error over all channels and pixels.
Here $x$ is the clean image after the input normalization used for training,
and $x_t$ is its noisy version at diffusion step $t$, obtained using Gaussian
noise $\epsilon$. The training checkpoint is indexed by $\tau$, and
$\phi(\tau)$ denotes the model parameters at that checkpoint.
Derivatives are taken with respect to $x$, including its contribution through
$x_t$. We estimate the input-Hessian trace using $32$ Monte Carlo draws per
image. Each draw samples $t$ uniformly from $\{1,\ldots,999\}$, Gaussian
forward noise $\epsilon$, and a Rademacher vector $v$, and evaluates
$v^\top\nabla_x^2\ell_\tau(x;t,\epsilon)v$.
The gradient curves average $\|\nabla_x\ell_\tau(x;t,\epsilon)\|_2$ over $32$
draws, and the loss curves average $\ell_\tau$ over $64$ draws. Training means
use all $n$ images; validation means use the fixed held-out set. These
measurements do not require generating samples from the model.

At each checkpoint used to evaluate memorization, we generate $10{,}000$
images using $100$ deterministic DDIM steps ($\eta=0$). A generated image is
classified as a copy when the ratio of its nearest to second-nearest
training-image distances is below $1/3$. We use Euclidean distances in the
normalized image space used for training.
The copy count $F(i,\tau)$ is the number of generated copies assigned to
training image $i$ at checkpoint $\tau$, and $f_{\mathrm{mem}}(\tau)$ is the
fraction of generations classified as copies.

\subsection{Supporting Figures}

Figure~\ref{fig:imagenet-ddpm-copied-groups} compares the mean curvature of
training images copied at a checkpoint with that of the remaining training
images. The copied group generally has higher curvature, although the
separation varies with training size and checkpoint.
Figure~\ref{fig:imagenet-ddpm-maxfreq-gap} compares the maximum copy count with
the relative train--validation curvature gap. Smaller training sets tend to
exhibit both larger maximum copy counts and larger curvature gaps.

% Keep these figures near the text and top-align any figure-only pages.
\begingroup
\makeatletter
\setlength{\@fptop}{0pt}
\makeatother
\begin{figure}[!htbp]
    \centering
    \includegraphics[width=\linewidth,height=0.76\textheight,keepaspectratio]{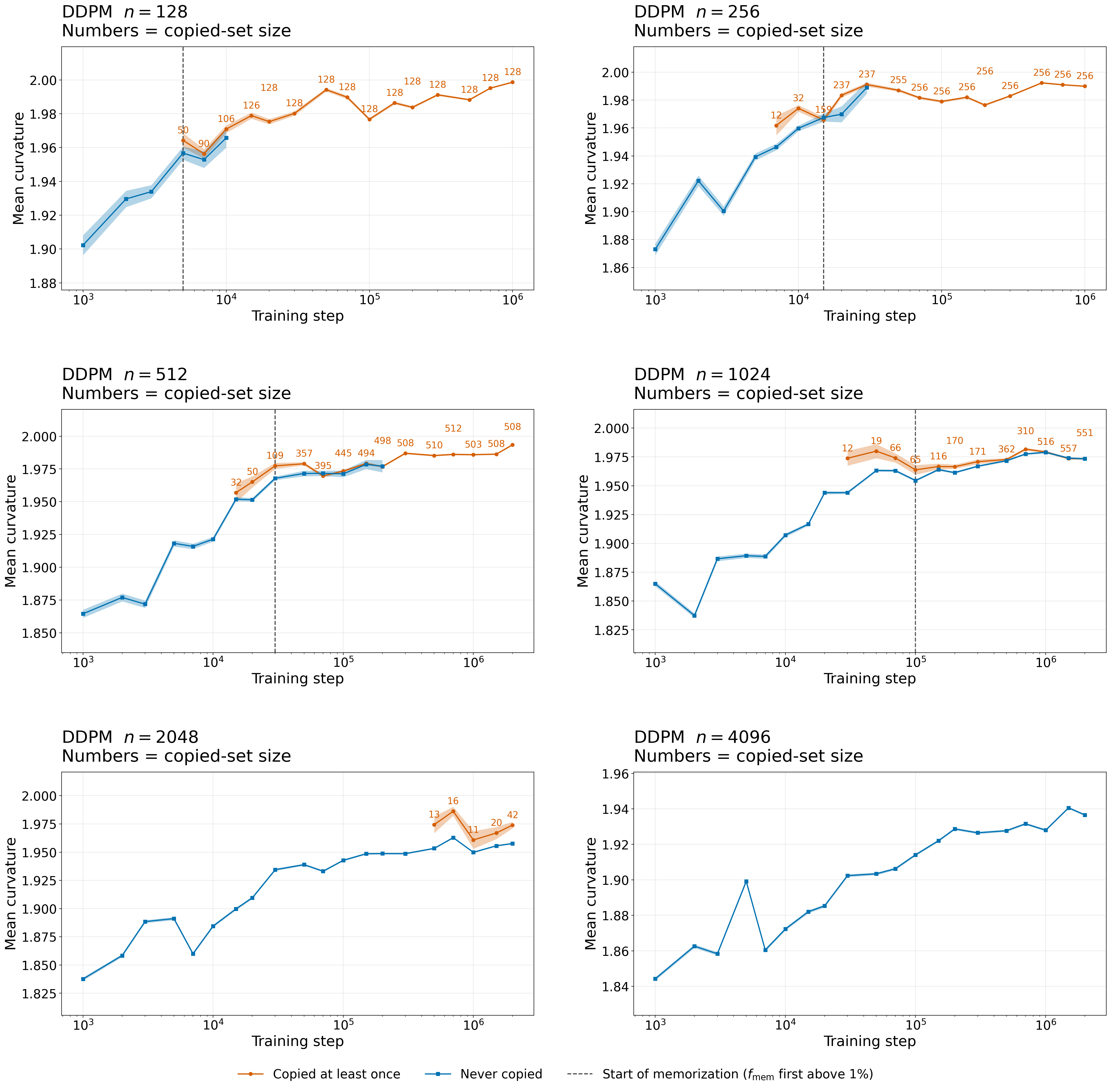}
    \caption{\textbf{Curvature of copied and non-copied training images.}
    Each panel shows one ImageNet DDPM training size. Orange and blue curves
    average the input-Hessian trace over images with $F(i,\tau)>0$ and
    $F(i,\tau)=0$, respectively. Group membership is determined separately
    at each checkpoint; ``never copied'' in the legend refers only to that
    checkpoint's generations. Shading shows one standard error across images
    in the group, and orange numbers give the copied-group size. A group is
    plotted only when it contains at least ten images. Dashed vertical lines
    mark the first evaluated checkpoint with $f_{\mathrm{mem}}>1\%$, used
    as a reference threshold.}
    \label{fig:imagenet-ddpm-copied-groups}
\end{figure}

\begin{figure}[!htbp]
    \centering
    \includegraphics[width=0.9\linewidth]{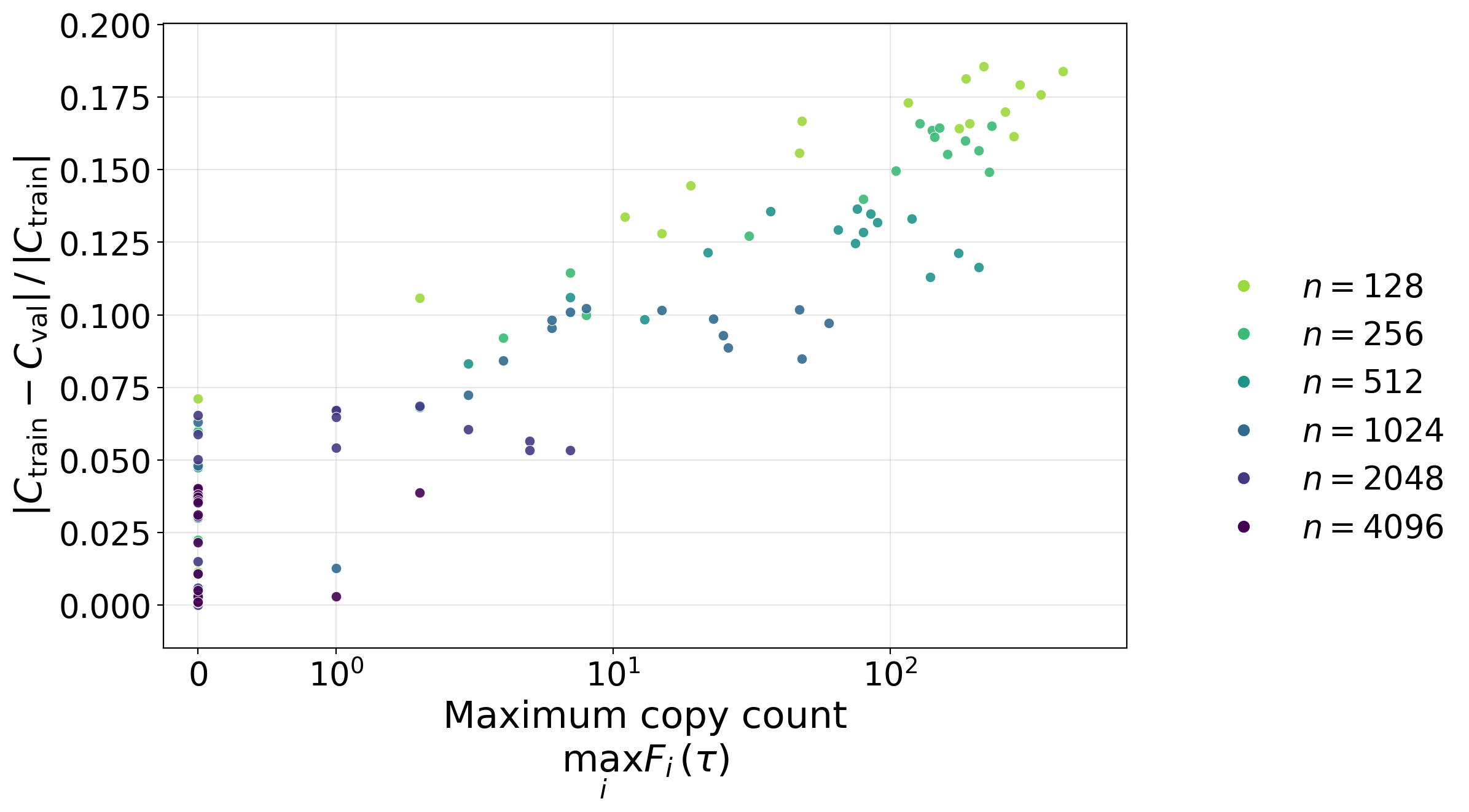}
    \caption{\textbf{Maximum copy count and relative curvature gap.}
    Each point represents one evaluated checkpoint of an ImageNet DDPM;
    colors indicate training-set size. The horizontal coordinate is
    $\max_i F(i,\tau)$ among $10{,}000$ generations. The vertical coordinate
    is $|C_{\mathrm{train}}(\tau)-C_{\mathrm{val}}(\tau)|/
    |C_{\mathrm{train}}(\tau)|$, where the two means use all training images
    and the same $1{,}024$ held-out images, respectively. Curvature is the
    input-Hessian trace estimated as described above. The horizontal scale
    is logarithmic above one and linear near zero to retain checkpoints
    with no observed copies.}
    \label{fig:imagenet-ddpm-maxfreq-gap}
\end{figure}
\clearpage
\endgroup

\end{document}